\documentclass[twoside]{article}

\usepackage[accepted]{aistats2019}
\usepackage[hang,flushmargin]{footmisc}
\usepackage[numbers]{natbib}

\usepackage{amsmath}
\usepackage{amsfonts}
\usepackage{amssymb}
\usepackage{amsthm}
\usepackage{bm}
\usepackage{bbm}
\usepackage{mathtools}
\usepackage{enumitem}
\usepackage{thmtools,thm-restate}
\usepackage{algorithm}
\usepackage{algorithmic}
\usepackage{natbib}
\usepackage{color}
\usepackage{graphicx}
\usepackage{comment}
\theoremstyle{plain}
\newtheorem{lemma}{Lemma}[section]

\theoremstyle{definition}

\theoremstyle{remark}

\def\AA{\mathcal{A}}

\def\TT{\mathcal{T}}

\def\Ebb{\mathbb{E}}

\def\Rbb{\mathbb{R}}

\def\R{\Rbb}

\def\diag{\mathrm{diag}}

\newcommand{\norm}[1]{ \| #1  \|  }

\newcommand{\lr}[2]{\langle #1, #2 \rangle}
\DeclareMathOperator*{\argmin}{arg\,min}

\newcommand{\E}{\Ebb}

\usepackage{etex,etoolbox}
\usepackage{environ}

\makeatletter
\providecommand{\@fourthoffour}[4]{#4}
\newcommand\fixstatement[2][\proofname\space of]{%
	\ifcsname thmt@original@#2\endcsname
	\AtEndEnvironment{#2}{%
		\xdef\pat@label{\expandafter\expandafter\expandafter
			\@fourthoffour\csname thmt@original@#2\endcsname\space\@currentlabel}%
		\xdef\pat@proofof{\@nameuse{pat@proofof@#2}}%
	}%
	\else
	\AtEndEnvironment{#2}{%
		\xdef\pat@label{\expandafter\expandafter\expandafter
			\@fourthoffour\csname #1\endcsname\space\@currentlabel}%
		\xdef\pat@proofof{\@nameuse{pat@proofof@#2}}%
	}%
	\fi
	\@namedef{pat@proofof@#2}{#1}%
}

\newcounter{proofcount}

\NewEnviron{proofatend}{%
	\edef\next{%
		\noexpand\begin{proof}[\pat@proofof\space\pat@label]%
			\unexpanded\expandafter{\BODY}}%
		\global\toks\numexpr\prooftoks+\value{proofcount}\relax=\expandafter{\next\end{proof}}
	\stepcounter{proofcount}}

\def\printproofs{%
	\count@=\z@
	\loop
	\the\toks\numexpr\prooftoks+\count@\relax
	\ifnum\count@<\value{proofcount}%
	\advance\count@\@ne
	\repeat}
\makeatother

\fixstatement{lemma}
\fixstatement{theorem}
\fixstatement{proposition}
\fixstatement{corollary}

\usepackage{microtype} % not sure whether we use this or not
\usepackage{graphicx}
\usepackage{rotating}
\usepackage{subfigure}
\usepackage{booktabs} % for professional tables
\usepackage{multirow}
\usepackage{caption}

\graphicspath{ {../figures/} }

\def\d{\text{\normalfont d}}

\newif\ifLONG
\LONGfalse
\ifLONG % just print the proof on spot in the long version

\fi

\theoremstyle{plain}

\newtheorem*{theorem*}{\bf Theorem}
\newtheorem*{prop*}{\bf Proposition}
\newtheorem*{lem*}{\bf Lemma}

\begin{document}



\twocolumn[

\aistatstitle{Truncated Back-propagation for Bilevel Optimization}

\aistatsauthor{Amirreza Shaban* \And Ching-An Cheng* \And  Nathan Hatch \And Byron Boots}

\aistatsaddress{Georgia Institute of Technology \quad *Equal contribution}]

\begin{abstract}
Bilevel optimization has been recently revisited for designing and analyzing algorithms in hyperparameter tuning and meta learning tasks. However, due to its nested structure, evaluating exact gradients for high-dimensional problems is computationally challenging.
One heuristic to circumvent this difficulty is to use the approximate gradient given by performing truncated back-propagation through the iterative optimization procedure that solves the lower-level problem. Although promising empirical performance has been reported, its theoretical properties are still unclear. 
In this paper, we analyze the properties of this family of approximate gradients and establish sufficient conditions for convergence. 
%We show that, a sufficient descent direction for the upper-level problem can be computed. 
We validate this on several hyperparameter tuning and meta learning tasks. We find that optimization with the approximate gradient computed using few-step back-propagation
often performs comparably to optimization with the exact gradient, while requiring far less memory and half the computation time.
\end{abstract}

\vspace{-2mm}
\section{INTRODUCTION}
\vspace{-1mm}

Bilevel optimization has been recently revisited as a theoretical framework for designing and analyzing algorithms for hyperparameter optimization~\citep{domke2012generic} and meta learning~\citep{franceschi2017bridge}.
Mathematically, these problems can be formulated as a stochastic optimization problem with an equality constraint (see Section \ref{sec:applications}):
\begin{align}  \label{eq:bilevel optimization}
\begin{split}
\min_{\lambda } 
F(\lambda) &:= \E_{S} \left[ f_S(\hat{w}_S^*(\lambda), \lambda) \right]    \\[-3pt]
\text{s.t.} \quad   \hat{w}_S^*(\lambda) &\approx_\lambda \argmin_{w } g_S(w, \lambda) 
\end{split}
\end{align} 
where $w$ and $\lambda$ are the \emph{parameter} and the \emph{hyperparameter}, 
$F$ and $f_S$ are the expected and the sampled \emph{upper-level objective}, $g_S$ is the sampled \emph{lower-level objective}, and $S$ is a random variable called the \emph{context}.
The notation $\approx_\lambda$ means that $\hat{w}_S^*(\lambda)$ equals the unique return value of a prespecified iterative algorithm (e.g. gradient descent) that approximately finds a local minimum of $g_S$.
This algorithm is part of the problem definition and can also be parametrized by $\lambda$ (e.g. step size).
The motivation to explicitly consider the approximate solution $\hat w_S^*(\lambda)$ rather than an exact minimizer $w_S^*$ of $g_S$ is that $w_S^*$ is usually not available in closed form.
This setup enables $\lambda$ to account for the imperfections of the lower-level optimization algorithm. % and to further improve it.

Solving the bilevel optimization problem in~\eqref{eq:bilevel optimization} is challenging due to the complicated dependency of the upper-level problem on $\lambda$ induced by $\hat{w}_S^*(\lambda)$. 
This difficulty is further aggravated when $\lambda$ and $w$ are high-dimensional, precluding the use of black-box optimization techniques such as grid/random search~\citep{bergstra2012random} and Bayesian optimization~\citep{srinivas2009gaussian,snoek2012practical}.

Recently, first-order bilevel optimization techniques have been revisited to solve these problems. These methods rely 
on an estimate of the Jacobian $\nabla_\lambda \hat{w}_S^*(\lambda)$ to optimize $\lambda$. %, because $\hat{w}_S^*(\lambda)$ is a function of $\lambda$. 
\citet{pedregosa2016hyperparameter} and \citet{gould2016differentiating} assume that $\hat{w}_S^*(\lambda) = w_S^*$ and compute $\nabla_\lambda \hat{w}_S^*(\lambda)$ % with respect to $\lambda$ 
by implicit differentiation. 
By contrast, \citet{maclaurin2015gradient} and \citet{franceschi2017forward} treat the iterative optimization algorithm in the lower-level problem as a dynamical system, and compute $\nabla_\lambda \hat{w}_S^*(\lambda)$ %the change of $\hat{w}^*(\lambda)$ with respect to $\lambda$ 
by automatic differentiation through the dynamical system. In comparison, the latter approach is less sensitive to the optimality of $\hat w_S^*(\lambda)$ and can also learn hyperparameters that control the lower-level optimization process (e.g. step size). 
However, due to superlinear time or space complexity (see Section \ref{sec:hypergradient}), neither of
%while these methods significantly scale the applicability of the bilevel optimization approach, 
these methods is applicable when both $\lambda$ and $w$ are high-dimensional~\citep{franceschi2017forward}.
%
%Due to the time and space complexity challenges of performing exact FMD and RMD, the applicability of bilevel-optimization models has been limited to smaller problems. % to the cases where either the hyperparameter $\lambda$ or the parameter $w$ is low-dimensional. 

Few-step reverse-mode automatic differentiation~\citep{luketina2016scalable,baydin2017online} and few-step forward-mode automatic differentiation~\citep{franceschi2017forward} have recently been proposed as heuristics to address this issue.  %to greedily update the hyperparameter. 
 By ignoring long-term dependencies, the time and space complexities to compute approximate gradients can be greatly reduced.
%While they have been empirically validated in several applications (e.g.learning-to-optimize and few-shot learning), 
While exciting empirical results have been reported, the theoretical properties of these methods remain unclear. 

%To scale meta learning to the cases where both $\lambda$ and the $w$ are high-dimensional, 
%To scale the bilevel optimization approach to problems where both $\lambda$ and the $w$ are high-dimensional, we consider updating $\lambda$ with approximate gradients. 
%\amirreza{We need to be more precise here. For general cases, we show convergence to approximate stationary point. We further show that for problems with more structure one step back-propagation is sufficient for the converge to the exact stationary point} 
In this paper, we study the theoretical properties of these \emph{truncated back-propagation} approaches.
We show that, when the lower-level problem is locally strongly convex around $\hat{w}_S^*(\lambda)$, on-average convergence to an $\epsilon$-approximate stationary point is guaranteed by $O(\log 1/\epsilon)$-step truncated back-propagation.
We also identify additional problem structures for which asymptotic convergence to an exact stationary point is guaranteed.
%optimizing $\lambda$ with such approximate gradient can be guaranteed 
Empirically, we verify the utility of this strategy for hyperparameter optimization and meta learning tasks.
We find that, compared to optimization with full back-propagation, optimization with truncated back-propagation usually shows competitive performance while requiring half as much computation time and significantly less memory.

%The goal is to find $\lambda$ such that $F(\lambda)$ is smalled when 

\vspace{-2mm}
\subsection{Applications}\label{sec:applications}
\vspace{-1mm}

\paragraph{Hyperparameter Optimization}

The goal of hyperparameter optimization \citep{larsen1996design,bengio2000gradient} is to find hyperparameters $\lambda$ for an optimization problem $P$ such that the approximate solution  $\hat w^*(\lambda)$ of $P$ has low cost $c(\hat w^*(\lambda))$ for some cost function $c$. 
In general, $\lambda$ can parametrize both the objective of $P$ and the algorithm used to solve $P$.
This setup is a special case of the bilevel optimization problem~\eqref{eq:bilevel optimization} where the upper-level objective $c$ does not depend directly on $\lambda$.
In contrast to meta learning (discussed below), $c$ can be deterministic \citep{franceschi2017forward}.
See Section~\ref{sec:hyperparameter optimization} for examples.

Many low-dimensional problems, such as choosing the learning rate and regularization constant for training neural networks, can be effectively solved with grid search.
However, problems with thousands of hyperparameters are increasingly common, for which gradient-based methods are more appropriate~\citep{maclaurin2015gradient, chen2014insights}.

\vspace{-1mm}
\paragraph{Meta Learning}
Another important application of bilevel optimization, meta learning (or learning-to-learn) uses statistical learning to optimize an algorithm $\AA_\lambda$ over a distribution of tasks $\TT$ and contexts $S$:
%due to an increasing need for algorithms that can exploit problem-specific information to achieve better performance, such as a faster convergence rate or smaller generalization error. 
\begin{align} \label{eq:meta learning}
\min_{\lambda } \E_{\TT} \E_{S | \TT} \left[ c_{\TT} ( \AA_{\lambda}(S)) \right]. 
\end{align}
It treats $\AA_{\lambda}$ as a parametric function, with hyperparameter $\lambda$, that takes task-specific context information $S$ as input and outputs a decision $\AA_{\lambda}(S)$. The goal of meta learning is to optimize the algorithm's performance $c_{\TT}$ (e.g. the generalization error) across tasks $\TT$ through empirical observations. This general setup subsumes multiple problems commonly encountered in the machine learning literature, such as multi-task learning~\citep{caruana1998multitask,ranjan2017hyperface} and few-shot learning~\citep{fei2006one,ravi2016optimization,snell2017prototypical}. 

%For instance, the context $S$ can be some training pairs in a supervised learning problem, and $\AA_{\lambda}$ denotes a training process controlled by $\lambda$ that returns a function approximator $\AA_{\lambda}(S)$ as output. 

Bilevel optimization emerges from meta learning when the algorithm computes $\AA_{\lambda}(S)$ by internally solving a \textit{lower-level} minimization problem with variable $w$.
% comment this if necessary
The motivation to use this class of algorithms is that the lower-level problem can be designed so that, even for tasks $\TT$ distant from the training set, $\AA_{\lambda}$ falls back upon a sensible optimization-based approach \citep{finn2017model,baydin2017online}. By contrast, treating $\AA_{\lambda}$ as a general function approximator relies on the availability of a large amount of meta training data~\citep{andrychowicz2016learning,li2017learning}.

In other words, the decision is $\AA_{\lambda}(S) = (\hat{w}_S^*(\lambda), \lambda)$ where $\hat w_S^*(\lambda)$ is an approximate minimizer of some function $g_S(w, \lambda)$.
Therefore, we can identify 
\begin{align} \label{eq:definition of f}
\E_{\TT | S} \left[ c_{\TT}(\hat{w}_S^*(\lambda), \lambda) \right] =: f_S(\hat{w}_S^*(\lambda), \lambda)
\end{align}
and write~\eqref{eq:meta learning} as~\eqref{eq:bilevel optimization}.\footnote{
We have replaced $\E_\TT \E_{S|\TT}$ with $\E_S \E_{\TT|S}$, which is valid since both describe the expectation over the joint distribution. The algorithm $\AA_{\lambda}$ only perceives $S$, not $\TT$.}
Compared with $\lambda$, the lower-level variable $w$ is usually task-specific and fine-tuned based on the given context $S$.
For example, in few-shot learning, a warm start initialization or regularization function ($\lambda$) can be learned through meta learning, so that a task-specific network ($w$) can be quickly trained using regularized empirical risk minimization with few examples $S$.
See Section \ref{sec:omniglot} for an example.

\vspace{-2mm}
\section{BILEVEL OPTIMIZATION}%{FIRST-ORDER APPROACH TO BILEVEL OPTIMIZATION}
\vspace{-2mm}
\subsection{Setup} \label{sec:setup}
\vspace{-1mm}
Let $\lambda \in \R^N$ and $w \in \R^M$.
We consider solving~\eqref{eq:bilevel optimization} with first-order methods that sample $S$ (like stochastic gradient descent) and focus on the problem of computing the gradients for a given $S$. Therefore, we will simplify the notation below by omitting the dependency of variables and functions on $S$ and $\lambda$ (e.g. we write $\hat{w}_S^*(\lambda)$ as $\hat{w}^*$ and $g_S$ as $g$). We use $\d_x$ to denote the total derivative with respect to a variable $x$, and $\nabla_x$ to denote the partial derivative,
with the convention that $\nabla_\lambda f \in \R^N$ and $\nabla_\lambda \hat{w}^* \in \R^{N \times M}$. 
 %For example, $ \nabla_{\hat{w}^*} f = [ \partial f/\partial \hat{w}^*_m ]_{m=1}^M \in \R^M $, and . 

To optimize $\lambda$, stochastic first-order methods use estimates of the gradient $\d_{\lambda} f = \nabla_{\lambda} f + \nabla_{\lambda} \hat{w}^* \nabla_{\hat{w}^*} f$.
Here we assume that both $\nabla_{\lambda} f \in \R^N$ and $\nabla_{\hat{w}^*} f \in \R^M$ are available through a stochastic first-order oracle, and focus on the problem of computing the matrix-vector product $\nabla_{\lambda} \hat{w}^*  \nabla_{\hat{w}^*} f $ when both $\lambda$ and $w $ are high-dimensional.

\vspace{-2mm}
\subsection{Computing the hypergradient}\label{sec:hypergradient}
\vspace{-1mm}
Like \citep{maclaurin2015gradient,franceschi2017forward}, we treat the iterative optimization algorithm that solves the lower-level problem as a dynamical system.
Given an initial condition $w_0 = \Xi_{0}(\lambda)$ at $t=0$, the update rule can be written as\footnote{For notational simplicity, we consider the case where $w_t$ is the state of~\eqref{eq:dynamics}; our derivation can be easily generalized to include other internal states, e.g. momentum.}
\begin{align}
 w_{t+1} = \Xi_{t+1}( w_t, \lambda  ), \qquad  \hat{w}^{*} = w_T
 \label{eq:dynamics}
\end{align} 
in which $\Xi_t$ defines the transition and and $T$ is the number iterations performed. For example, in gradient descent, $\Xi_{t+1} (w_t, \lambda) = w_t - \gamma_t (\lambda) \nabla_{w} g (w_t, \lambda) $, where 
$\gamma_t (\lambda)$ is the step size.

By unrolling the iterative update scheme~\eqref{eq:dynamics} as a computational graph, we can view $\hat{w}^{*}$ as a function of $\lambda$ and compute the required derivative $\d_{\lambda} f $~\citep{baydin2015automatic}. Specifically, it can be shown by the chain rule\footnote{Note that this assumes $g$ is twice differentiable.}
\begin{align} \textstyle
 \d_{\lambda} f  =   \nabla_\lambda f  + \sum_{t=0}^{T}  B_{t} A_{t+1} \cdots A_{T}  \nabla_{\hat{w}^*} f \label{eq:unrolling}
\end{align}
where $A_{t+1} = \nabla_{w_t} \Xi_{t+1}(w_t, \lambda)$,  $B_{t+1} = \nabla_{\lambda} \Xi_{t+1}(w_t, \lambda )$  for $t \geq 0$, and $ B_0 = \d_{\lambda} \Xi_0 (\lambda) $.

The computation of~\eqref{eq:unrolling} can be implemented either in reverse mode or forward mode~\citep{franceschi2017forward}. 
Reverse-mode differentiation (RMD) computes~\eqref{eq:unrolling} by  back-propagation:
\begin{equation} \label{eq:RMD}
\begin{split}
\alpha_{T} &= \nabla_{\hat{w}^*} f , \quad h_{T} = \nabla_\lambda f,  \\[-1pt]
h_{t-1} &= h_{t} + B_{t} \alpha_{t} ,  \quad \alpha_{t-1} = A_{t} \alpha_{t}
\end{split}
\end{equation}
and finally $\d_\lambda f  = h_{-1}$. Forward-mode differentiation (FMD) computes~\eqref{eq:unrolling} by forward propagation:
\begin{align} \label{eq:FMD}
\begin{split} 
Z_0 &= B_0, \quad Z_{t+1} =  Z_t A_{t+1} + B_{t+1}, \\[-1pt]
 \d_\lambda f  &= Z_{T} \nabla_{\hat{w}^*} f + \nabla_\lambda f
\end{split}
\end{align}

\begin{table}[t]
  \caption{Comparison of the additional time and space to compute $\d_{\lambda} f = \nabla_{\lambda} f + \nabla_{\lambda} \hat{w}^* \nabla_{\hat{w}^*} f $, 
  where $\lambda \in \R^N$, $w~\in~\R^M$, and $c = c(M,N)$ is the time complexity to compute the transition function $\Xi$. 
  $^\dagger$Checkpointing doubles the constant in time complexity, compared with other approaches.}
\label{table:complexity-comparison} 
%\vskip 0.15in
%\vspace{-2mm}
\begin{center}
\begin{small}
\begin{sc}
\begin{tabular}{lcccr}
\toprule
Method & Time & Space & Exact \\
\midrule 
FMD & $O(cNT)$ & $O(MN)$ & \checkmark   \\
RMD & $O(cT)$  & $O(MT)$ & \checkmark   \\
Checkpointing & $O(cT^\dagger)$ & $O(M\sqrt{T})$ & \checkmark \\
every $\sqrt{T}$ steps$^\dagger$  \\
$K$-RMD & $O(cK)$  & $O(MK)$ &  \\
\bottomrule
\end{tabular}
\end{sc}
\end{small}
\end{center}
\vskip -2mm
\end{table}
The choice between RMD and FMD is a trade-off based on the size of $w \in \mathbb{R}^M$ and $\lambda \in \mathbb{R}^N$ (see Table~\ref{table:complexity-comparison} for a comparison). 
For example, one drawback of RMD is that all the intermediate variables $\{w_t \in \R^M\}_{t=1}^T$ need to be stored in memory in order to compute $A_t$ and $B_t$ in the backward pass. Therefore, RMD is only applicable when $MT$ is small, as in~\citep{finn2017model}. 
Checkpointing~\citep{hascoet2006enabling} can reduce this to $M\sqrt{T}$, but it \emph{doubles} the computation time.
%. In addition, because of exploding/vanishing gradient problem, numerical instability can happen if $T$ is large~\citep{maclaurin2015gradient}. 
Complementary to RMD, FMD propagates the matrix $Z_t \in \R^{M \times N}$ in line with the forward evaluation of the dynamical system~\eqref{eq:dynamics}, and does not require any additional memory to save the intermediate variables.
However, propagating the matrix $Z_t$ instead of vectors requires memory of size $MN$ and is $N$-times slower compared with RMD.

\vspace{-2mm}
\section{TRUNCATED BACK-PROPAGATION\label{sec:theory}}
\vspace{-1mm}
In this paper, we investigate approximating~\eqref{eq:unrolling} with partial sums, which was previously proposed as a heuristic for bilevel optimization (\citep{luketina2016scalable} Eq. 3, \citep{baydin2017online} Eq. 2). Formally, we perform $K$-step truncated back-propagation ($K$-RMD) and use the intermediate variable $h_{T-K}$ to construct an approximate gradient:
\begin{align} \textstyle
h_{T-K} =  \nabla_\lambda f + \sum_{t=T-K+1}^{T}  B_{t} A_{t+1} \cdots A_{T}  \nabla_{\hat{w}^*} f \label{eq:incomplete RMD}
\end{align}
This approach requires storing only the last $K$ iterates $w_t$, and it also saves computation time. Note that $K$-RMD can be combined with checkpointing for further savings, although we do not investigate this.

\vspace{-2mm}
\subsection{General properties}
\vspace{-1mm}
%We show that, under certain conditions, $-h_{T-K}$ is a sufficient descent direction, and therefore performing first-order optimization with $h_{T-K}$ can have convergence guarantees. 

%The hope is that, when the problem is well behaved, even using this truncated sum can generate a good gradient estimate such that  $\d_{\lambda} f \approx  h_{T-K}$. 

%In the following sections, we first expand on this idea and study the properties of $h_{T-K}$. Next, we present two algorithms to optimize $\lambda$: one uses $h_{T-K}$ as a biased gradient estimate, and  the other uses $h_{T-K}$ as a control variate to construct an unbiased gradient estimate. 
%While our algorithm requires $K$-times more computation and memory compared with one-step heuristics, using our $K$-step incomplete RMD posses theoretical guarantees: choosing $K = O(\log \epsilon)$ is sufficient to convergence to an $\epsilon$-approximate stationary point on average.  

%\subsection{Properties of approximate gradients} \label{sec:incomplete RMD}

%While the approximate gradient $g_{T-k}$ is biased,

We first establish some intuitions about why using $K$-RMD to optimize $\lambda$ is reasonable. 
While building up an approximate gradient by truncating back-propagation in general optimization problems can lead to large bias, the bilevel optimization problem in~\eqref{eq:bilevel optimization} has some nice structure. 
Here we show that if the lower-level objective $g$ is locally strongly convex around $\hat{w}^*$, then the bias of $h_{T-K}$ can be exponentially small in $K$. That is, choosing a small $K$ would suffice to give a good gradient approximation in finite precision. The proof is given in Appendix~\ref{app:proof of exp convergence}.
\begin{restatable}{prop}{expConvergence}\label{pr:exp convergence}
Assume $g$ is  $\beta$-smooth, twice differentiable, and locally $\alpha$-strongly convex in $w$ around $\{w_{T-K-1}, \dots, w_T\}$.
Let %$\Xi$ be the gradient descent iteration with step size $\gamma$, i.e.
$
	\Xi_{t+1}(w_t, \lambda) = w_t - \gamma \nabla_{w} g(w_t, \lambda)
$. 
For $\gamma \leq \frac{1}{\beta}$, % \leq \frac{1}{\alpha}$, 
it holds
\begin{align}
\norm{h_{T-K} - \d_{\lambda}  f }\leq  2^{T-K+1}  (1 -\gamma \alpha)^K  \norm{\nabla_{\hat{w}^*} f}  M_B   \label{eq:nonconvex bound}
\end{align}
where  $M_B = \max_{t\in\{0,\dots,T-K\} } \norm{B_t}$.
In particular, if $g$ is globally $\alpha$-strongly convex, then 
\begin{align} \textstyle
\norm{h_{T-K} - \d_{\lambda}  f } \leq \frac{(1 -\gamma \alpha)^K }{\gamma \alpha}  \norm{\nabla_{\hat{w}^*} f} M_B.    \label{eq:convex bound}
\end{align}
\end{restatable}\vspace{-1mm}
Note $0 \leq (1 -\gamma \alpha) < 1$ since $\gamma \leq \frac1\beta \leq \frac1\alpha$. Therefore, Proposition~\ref{pr:exp convergence} says that if $\hat{w}^*$ converges to the \emph{neighborhood} of a strict local minimum of the lower-level optimization, then the bias of using the approximate gradient of $K$-RMD decays exponentially in $K$. 
This exponentially decaying property is the main reason why using $h_{T-K}$ to update the hyperparameter $\lambda$ works.

Next we show that,  when the lower-level problem $g$ is second-order continuously differentiable, $-h_{T-K}$ actually is a sufficient descent direction. This is a much stronger property than the small bias shown in Proposition~\ref{pr:exp convergence}, and it is critical in order to prove convergence to exact stationary points (cf. Theorem~\ref{th:convergence of K-step backprop}).
To build intuition, here we consider a simpler problem where $g$ is globally strongly convex and $\nabla_{\lambda} f  = 0$. These assumptions will be relaxed in the next subsection. 

\begin{restatable}{lem}{boundOfManyTerms} \label{lm:bound of many terms}
  Let $g$ be globally strongly convex and $\nabla_{\lambda} f = 0$.  Assume $g$ is second-order continuously differentiable and $B_t$ has full column rank for all $t$. Let $
  \Xi_{t+1}(w_t, \lambda) = w_t - \gamma \nabla_{w} g(w_t,\lambda)
  $. 
For all $K\geq 1$, with $T$ large enough and $\gamma$ small enough, there exists $c > 0$, s.t.
	%\begin{align*}
	$h_{T-K}^\top \d_{\lambda} f \geq  c  \norm{\nabla_{\hat{w}^*} f }^2$. 
	%\end{align*}
	This implies $h_{T-K} $ is a sufficient descent direction, i.e. $	h_{T-K}^\top \d_{\lambda}f   \geq \Omega (\norm{\d_{\lambda} f}^2)$.
\end{restatable}\vspace{-2mm}
%This lemma is the main technical contribution of the paper.
The full proof of this non-trivial result is given in Appendix~\ref{app:proof of bound of many terms}. Here we provide some ideas about why it is true. 
First, by Proposition~\ref{pr:exp convergence}, we know the bias decays exponentially. However, this alone is not sufficient to show that $-h_{T-K}$ is a sufficient descent direction. 
To show the desired result, Lemma~\ref{lm:bound of many terms} relies on the assumption that $g$ is second-order continuously differentiable and the fact that using gradient descent to optimize a well-conditioned function has linear convergence~\citep{hazan2016introduction}. These two new structural properties further reduce the bias in Proposition~\ref{pr:exp convergence} and lead to Lemma~\ref{lm:bound of many terms}.
Here the full rank assumption for $B_t$ is made to simplify the proof. We conjecture that this condition can be relaxed when $K>1$. We leave this to future work.

\vspace{-2mm}
\subsection{Convergence}
\vspace{-1mm}

With these insights, we analyze the convergence of bilevel optimization with truncated back-propagation. 
%%Let $F(\lambda) = \E_S \left[ f(\hat{w}^*(\lambda, S),\lambda;S) \right]$ be the upper-level objective.
%Next, we relax the assumption on $S$, to consider general stochastic problems.
Using Proposition~\ref{pr:exp convergence}, we can immediately deduce that optimizing $\lambda$ with $h_{T-K}$ converges on-average to an $\epsilon$-approximate stationary point. 
Let $\nabla F(\lambda_\tau)$ denote the hypergradient in the $\tau$th iteration.

\begin{restatable}{theorem}{biasedConvergence} \label{th:biased convergence}	
	Suppose $F$ is smooth and bounded below, and suppose there is $\epsilon< \infty$ such that $\norm{h_{T-K} - \d_{\lambda}  f } \leq \epsilon$.
	Using $h_{T-K}$ as a stochastic first-order oracle with a decaying step size $\eta_{\tau} = O(1/\sqrt{\tau}) $  to update $\lambda$ with gradient descent, it follows after $R$ iterations,
	\begin{align*}
	%\textstyle
	\E\left[ \sum_{\tau=1}^{R}  \frac{ \eta_{\tau} \norm{  \nabla F(\lambda_\tau)  }^2}{ \sum_{\tau=1}^{R} \eta_{\tau}} \right] \leq  \widetilde O \left(  \epsilon +  \frac{\epsilon^2+1}{\sqrt{R}}\right).
	\end{align*}
	That is, under the assumptions in Proposition~\ref{pr:exp convergence}, learning with $h_{T-K}$ converges to an $\epsilon$-approximate stationary point, where $\epsilon = O((1-\gamma\alpha)^{-K})$. 	
\end{restatable}\vspace{-1mm}
We see that the bias becomes small as $K$ increases. As a result, it is sufficient to perform $K$-step truncated back-propagation with $K = O(\log 1/\epsilon)$ to update $\lambda$.
%and then use $h_{T-K}$ as an approximate first-order oracle 

Next, using Lemma~\ref{lm:bound of many terms}, we show that the bias term in Theorem~\ref{th:biased convergence} can be removed if the problem is more structured. As promised, we relax the simplifications made 
in Lemma~\ref{lm:bound of many terms} into assumptions 2 and 3 below and only assume $g$ is locally strongly convex.
%Compared with Lemma~\ref{lm:bound of many terms}, here we relax the assumption that $g$ is globally strongly convex and $\nabla_{\lambda} f  = 0$ and rely instead on local strong convexity.
\begin{restatable}{theorem}{convKStep} \label{th:convergence of K-step backprop}
	Under the assumptions in Proposition~\ref{pr:exp convergence} and Theorem~\ref{th:biased convergence}, if in addition
	\begin{enumerate}[topsep=0pt,itemsep=-1pt]
%		\item  $g$ is locally strongly convex around $w_T$ 
		\item  $g$ is second-order continuously differentiable
    	\item  $B_t$ has full column rank around $w_T$
		\item  $ \nabla_\lambda f^\top (\d_\lambda f + h_{T-K} -\nabla_\lambda f )  
		\geq \Omega(\norm{\nabla_\lambda f}^2)$
		%$\nabla_\lambda f^\top (d_\lambda f - \nabla_\lambda f ) \geq 0 $ \cheng{Update this}
		\item  the problem is deterministic (i.e. $F = f$)
	\end{enumerate}
	then for all $K\geq 1$, with $T$ large enough and $\gamma$ small enough,
	 the limit point is an exact stationary point, i.e. $
	\lim_{\tau\to\infty} \norm{ \nabla F(\lambda_\tau)  }= 0
	$.
\end{restatable}\vspace{-1mm}
%First, we consider the special case when $S$ is deterministic (i.e. $F = f$); for example, when \eqref{eq:bilevel optimization} corresponds to an empirical risk minimization problem. 

Theorem~\ref{th:convergence of K-step backprop} shows that if the partial derivative $\nabla_\lambda f$ does not interfere strongly with the partial derivative computed through back-propagating the lower-level optimization procedure (assumption 3), then optimizing $\lambda$ with $h_{T-K}$ converges to an \emph{exact} stationary point. 
%the bias term in the on-average convergence in Theorem~\ref{th:biased convergence} can be removed. 
This is a very strong result for an interesting special case. It shows that even with one-step back-propagation $h_{T-1}$, updating $\lambda$ can converge to a stationary point.

This non-interference assumption unfortunately is necessary; otherwise, truncating the full RMD leads to constant bias, as we show below (proved in Appendix~\ref{app:proof of counterexample}).

\begin{restatable}{theorem}{counterExample} \label{th:counterexample}
	There is a problem, satisfying all but assumption 3 in Theorem~\ref{th:convergence of K-step backprop}, such that optimizing $\lambda$ with $h_{T-K}$ does not converge to a stationary point.
\end{restatable}\vspace{-2mm}

Note however that the non-interference assumption is satisfied when $\nabla_{\lambda} f = 0$, i.e. when the upper-level problem does not directly depend on the hyperparameter.
This is the case for many practical applications: e.g. hyperparameter optimization, meta-learning regularization models, image desnosing \citep{roth2005fields, chen2014insights}, data hyper-cleaning \citep{franceschi2017forward}, and task interaction~\citep{evgeniou2005learning}.

%In particular, when $S$ is deterministic, 
\vspace{-1mm}
\subsection{Relationship with implicit differentiation}
\vspace{-1mm}
The  gradient estimate $h_{T-K}$ is related to implicit differentiation, which is a classical first-order approach to solving bilevel optimization problems~\citep{larsen1996design,bengio2000gradient}. %~\citep{gould2016differentiating}. 
Assume $g$ is second-order continuously differentiable and that its optimal solution uniquely exists such that $w^* = w^*(\lambda)$. By the implicit function theorem~\citep{rudin1964principles}, the total derivative of $f$ with respect to $\lambda$ can be written as 
\begin{align}
  \d_{\lambda} f = \nabla_{\lambda} f   - \nabla_{\lambda, w} g \nabla_{w, w}^{-1} g  \nabla_{\hat{w}^*} f \label{eq:implicit differentiation}
\end{align}
where all derivatives are evaluated at $(w^*(\lambda), \lambda)$ and $\nabla_{\lambda, w} g  = \nabla_{\lambda} (\nabla_{w} g)\in \R^{N\times M}$. 

Here we show that, in the limit where $\hat{w}^*$ converges to $w^*$,  $h_{T-K}$ can be viewed as approximating the matrix inverse in~\eqref{eq:implicit differentiation} with an order-$K$ Taylor series. This can be seen from the next proposition.
\begin{restatable}{prop}{matrixTaylor} \label{pr:taylor}
Under the assumptions in Proposition~\ref{pr:exp convergence}, suppose $w_t$ converges to a stationary point $w^*$. Let  $ A_{\infty} = \lim_{t\to\infty} A_t $ and $ B_{\infty} = \lim_{t\to\infty} B_t $. For $\gamma < \frac{1}{\beta}$, it satisfies that 
\begin{align} \textstyle
- \nabla_{\lambda, w} g \nabla_{w, w}^{-1} g   = B_\infty  \sum_{k=0}^{\infty} A_{\infty}^k  \label{eq:infinte series}
\end{align}
\end{restatable}% \vspace{-3mm}
By Proposition~\ref{pr:taylor}, we can write $\d_{\lambda}f $ in \eqref{eq:implicit differentiation} as
\begin{align*} \textstyle
\d_{\lambda}f
&= \nabla_{\lambda} f - \nabla_{\lambda, w} g \nabla_{w, w}^{-1} g \nabla_{\hat{w}^*} f \\
&=  h_{T-K} +  B_\infty  \sum_{k=K}^{\infty} A_{\infty}^k  \nabla_{\hat{w}^*} f
\end{align*}
That is, $h_{T-K}$ captures the first $K$ terms in the Taylor series, and the residue term has an upper bound as in Proposition~\ref{pr:exp convergence}.

Given this connection, we can compare the use of $h_{T-K}$ and approximating~\eqref{eq:implicit differentiation} using $K$ steps of conjugate gradient descent for high-dimensional problems~\citep{pedregosa2016hyperparameter}. First, both approaches require local strong-convexity to ensure a good approximation. Specifically, let $\kappa = \frac{\beta}{\alpha} > 0$ locally around the limit.
Using $h_{T-K}$ has a bias in $O((1-\frac1\kappa)^K)$, whereas using \eqref{eq:implicit differentiation} and inverting the matrix with $K$ iterations of conjugate gradient has a bias in $O((1-\frac{1}{\sqrt{\kappa}})^{K})$~\citep{shewchuk1994introduction}. Therefore, when $w^*$ is available, solving~\eqref{eq:implicit differentiation} with conjugate gradient descent is preferable. However, in practice, this is hardly true. When an approximate solution $\hat{w}^*$ to the lower-level problem is used, adopting~\eqref{eq:implicit differentiation} has no control on the approximate error, nor does it necessarily yield a descent direction. On the contrary, $h_{T-K}$ is based on Proposition~\ref{pr:exp convergence}, which uses a weaker assumption and does not require the convergence of $w_t$ to a stationary point.
Truncated back-propagation can also optimize the hyperparameters that control the lower-level optimization process, which the implicit differentiation approach cannot do.

\vspace{-2mm}
\section{EXPERIMENTS}
\vspace{-1mm}
\subsection{Toy problem}
\label{sec:toy}

Consider the following simple problem for $\lambda, w \in \R^2$:
\begin{align*}
  & \min_\lambda \norm{\hat w^*}^2 + 10\norm{\sin(\hat w^*)}^2 =: f(\hat w^*, \lambda) \\
  \text{s.t. } & \hat w^* \approx \argmin_w \textstyle\frac12 (w-\lambda)^\top G (w-\lambda) =: g(w, \lambda)
\end{align*}
where $\norm{\cdot}$ is the $\ell_2$ norm, sine is applied elementwise, $G = \diag(1, \frac12)$, and 
we define $\hat w^*$ as the result of $T=100$ steps of gradient descent on $g$ with learning rate $\gamma=0.1$, initialized at $w_0 = (2,2)$.
A plot of $f(\cdot, \lambda)$ is shown in Figure.~\ref{fig:toy1}.
We will use this problem to visualize the theorems and explore the empirical properties of truncated back-propagation.

This deterministic problem satisfies all of the assumptions in the previous section, particularly those of Theorem~\ref{th:convergence of K-step backprop}: $g$ is $1$-smooth and $\frac12$-strongly convex, with
\[
B_{t+1} = \nabla_\lambda [w_t - \gamma \nabla_w g(w_t, \lambda) ] = \gamma G
\]
and $B_0 = 0$. Although $f$ is somewhat complicated, with many saddle points, it satisfies the non-interference assumption because $\nabla_\lambda f = 0$.

\begin{figure}[t]
  \includegraphics[width=0.23\textwidth]{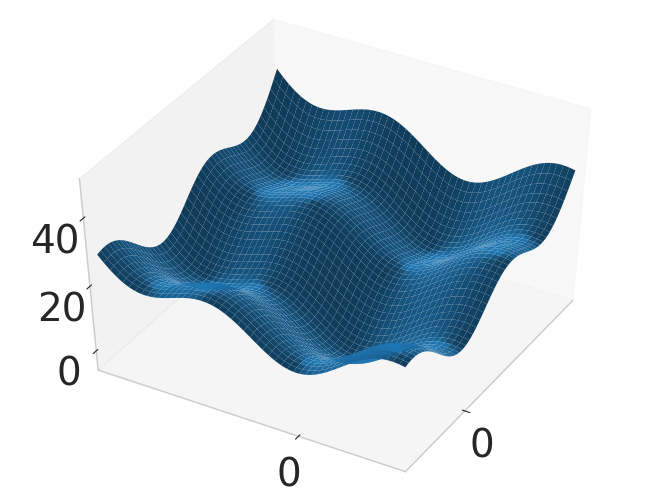}
  \includegraphics[width=0.23\textwidth]{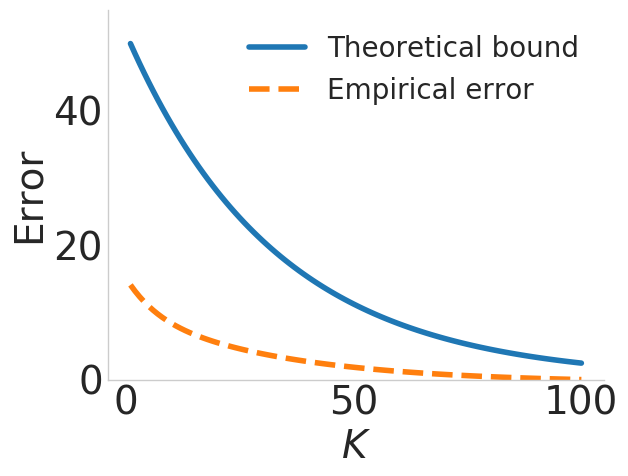}
  \caption{Graph of $f$ and visualization of Prop. \ref{pr:exp convergence}.}
  \label{fig:toy1}
  \vspace{-3mm}
\end{figure}
\begin{figure}
	\includegraphics[width=0.23\textwidth]{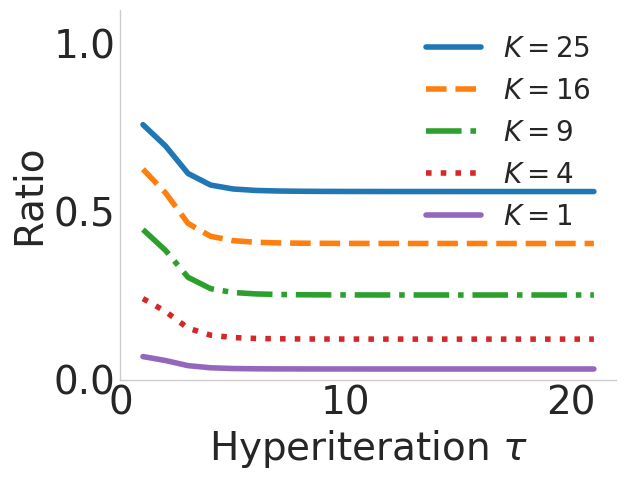}
	\includegraphics[width=0.23\textwidth]{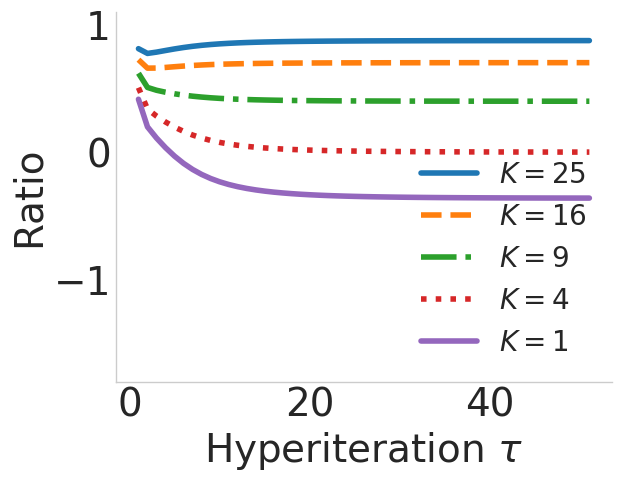}
	\caption{The ratio $h_{T-K}^\top \d_\lambda f / \norm{\d_\lambda f}^2$ at various $\lambda_\tau$, for $f$ and $\widetilde f$ respectively.}
	\label{fig:toy2}
	\vspace{-3mm}
\end{figure}
Figure \ref{fig:toy1} visualizes Proposition~\ref{pr:exp convergence} by plotting the approximation error $\norm{h_{T-K} - \d_\lambda f}$ and the theoretical bound $\frac{(1-\gamma\alpha)^K}{\gamma\alpha}\norm{\nabla_{\hat w^*}f}M_B$ at $\lambda=(1,1)$. %for comparison.
For this problem, $\alpha = \frac12$, $M_B = \norm{\gamma G} = \gamma$, and $\nabla_{\hat w^*} f$ can be found analytically from $\hat w^* = Cw_0 + (I-C)\lambda$, where $C = (I-\gamma G)^T$.
Figure \ref{fig:toy3} (left) plots the iterates $\lambda_\tau$ when optimizing $f$ using $1$-RMD and a decaying meta-learning rate $\eta_\tau = \frac{\eta_0}{\sqrt \tau}$.\footnote{
Because $\norm{h_{T-K}}$ varies widely with $K$, we tune $\eta_0$ to ensure that the first update $\eta_1 h_{T-K}(\lambda_1)$ has norm $0.6$.}
In comparison with the true gradient $\d_\lambda f$ at these points, we see that $h_{T-1}$ is indeed a descent direction.
Figure \ref{fig:toy2} (left) visualizes this in a different way, by plotting $h_{T-K}^\top \d_\lambda f / \norm{\d_\lambda f}^2$ for various $K$ at each point $\lambda_\tau$ along the $K=1$ trajectory.
By Lemma~\ref{lm:bound of many terms}, this ratio stays well away from zero.

To demonstrate the biased convergence of Theorem~\ref{th:biased convergence}, we break assumption 3 of
Theorem~\ref{th:convergence of K-step backprop}
by changing the upper objective to $\widetilde f(\hat w^*, \lambda) := f(\hat w^*, \lambda) + 5\norm{\lambda - (1,0)}^2$
so that $\nabla_\lambda \widetilde f \neq 0$.
The guarantee of Lemma \ref{lm:bound of many terms} no longer applies, and we see in Figure \ref{fig:toy2} (right) that $h_{T-K}^\top \d_\lambda f / \norm{\d_\lambda f}^2$ can become negative.
Indeed, Figure \ref{fig:toy4} shows that optimizing $\widetilde f$ with $h_{T-1}$ converges to a suboptimal point.
However, it also shows that using larger $K$ rapidly decreases the bias.

For the original objective $f$,
Theorem~\ref{th:convergence of K-step backprop}
guarantees exact convergence.
Figure \ref{fig:toy3} shows optimization trajectories for various $K$, and a log-scale plot of their convergence rates.
Note that, because the lower-level problem cannot be perfectly solved within $T$ steps, the optimal $\lambda$ is offset from the origin.
Truncated back-propagation can handle this, but it breaks the assumptions required by the implicit differentiation approach to bilevel optimization.

\begin{figure}
  \includegraphics[width=0.23\textwidth]{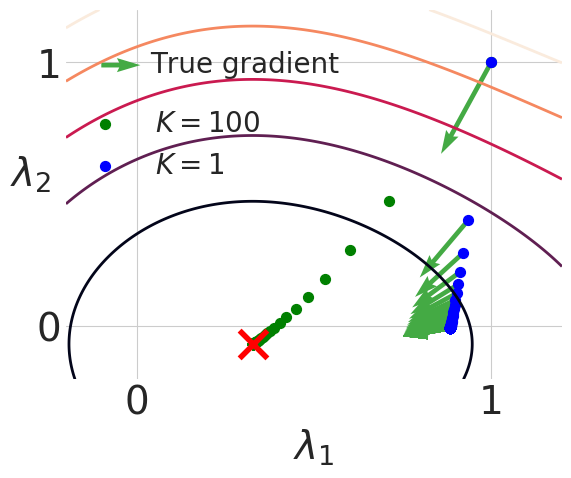}
  \includegraphics[width=0.23\textwidth]{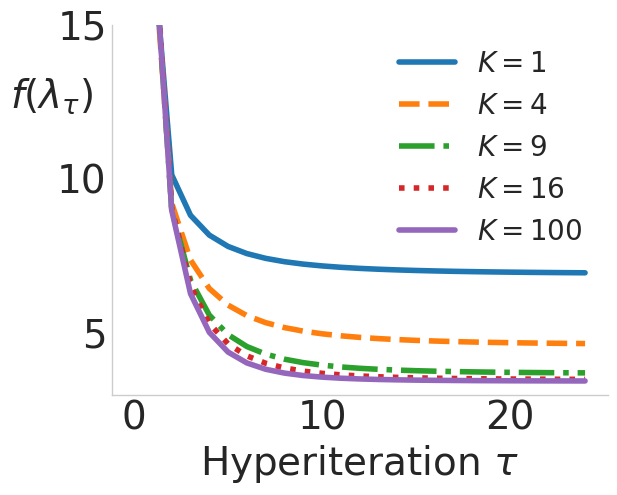}
  \caption{Biased convergence for $\widetilde f$. The red X marks the optimal $\lambda$.}
  \label{fig:toy4}
  \vspace{-3mm}
\end{figure}
\begin{figure}
	\includegraphics[width=0.23\textwidth]{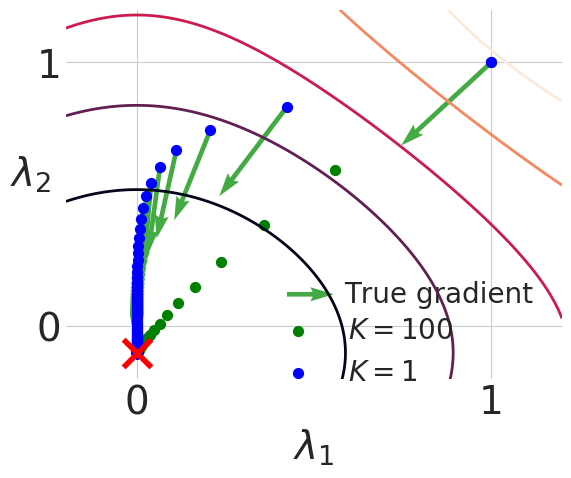}
	\includegraphics[width=0.23\textwidth]{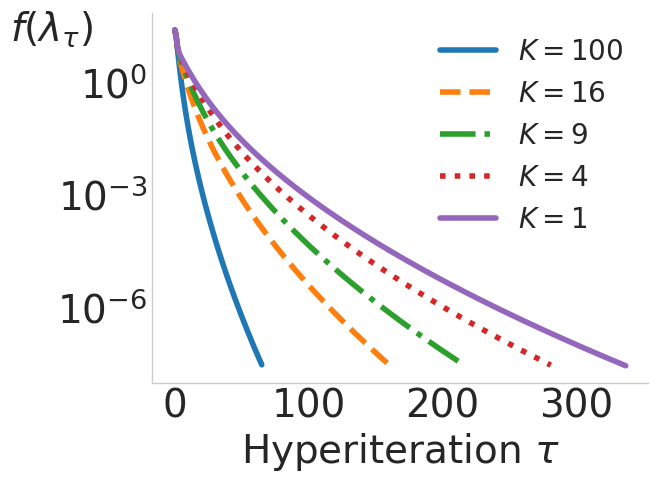}
	\caption{Convergence for $f$.}
	\label{fig:toy3}
  \vspace{-3mm}
\end{figure}

\vspace{-2mm}
\subsection{Hyperparameter optimization problems} \label{sec:hyperparameter optimization}
\vspace{-1mm}
\subsubsection{Data hypercleaning}
\label{sec:mnist}
  \vspace{-1mm}
In this section, we evaluate $K$-RMD on a hyperparameter optimization problem.
The goal of data hypercleaning \citep{franceschi2017forward} is to train a linear classifier for MNIST~\cite{lecun1998gradient}, with the complication that half of our training labels have been corrupted.
To do this with hyperparameter optimization, let $W \in \R^{10\times 785}$ be the weights of the classifier, with the outer objective $f$ measuring the cross-entropy loss on a cleanly labeled validation set.
The inner objective is defined as \emph{weighted} cross-entropy training loss plus regularization:
\[\textstyle
g(W, \lambda) = \sum_{i=1}^{5000} -\sigma(\lambda_i) \log(e_{y_i}^\top W x_i) + 0.001\norm{W}_F^2
\]
where $(x_i, y_i)$ are the training examples, $\sigma$ denotes the sigmoid function, $\lambda_i \in \R$, and $\norm{\cdot}_F$ is the Frobenius norm.
We optimize $\lambda$ to minimize validation loss, presumably by decreasing the weight of the corrupted examples.
The optimization dimensions are $|\lambda| = 5000$, $|W| = 7850$.
\citet{franceschi2017forward} previously solved this problem with full RMD, and it happens to satisfy many of our theoretical assumptions,
making it an interesting case for empirical study.\footnote{
We have reformulated the constrained problem from \citep{franceschi2017forward} as an unconstrained one that more closely matches our theoretical assumptions.
For the same reason, we regularized $g$ to make it strongly convex.
Finally, we do not retrain on the hypercleaned training + validation data.
This is because, for our purposes, comparing the performance of $\hat w^*$ across $K$ is sufficient.}

We optimize the lower-level problem $g$ through $T=100$ steps of gradient descent with $\gamma=1$ and consider how adjusting $K$ changes the performance of $K$-RMD.\footnote{
See Appendix \ref{sec:mnist_appendix} for more experimental setup.}
Our hypothesis is that $K$-RMD for small $K$ works almost as well as full RMD in terms of validation and test accuracy,
while requiring less time and far less memory.
We also hypothesize that $K$-RMD does almost as well as full RMD in identifying which samples were corrupted~\citep{franceschi2017forward}.
Because our formulation of the problem is unconstrained, the weights $\sigma(\lambda_i)$ are never exactly zero.
However, we can calculate an F1 score by setting a threshold on $\lambda$: if $\sigma(\lambda_i) < \sigma(-3) \approx 0.047$, then the hyper-cleaner has marked example $i$ as corrupted.\footnote{F1 scores for other choices of the threshold were very similar. See Appendix \ref{sec:mnist_appendix} for details.}

Table \ref{table:mnist} reports these metrics for various $K$.
We see that $1$-RMD is somewhat worse than the others, and that validation loss (the outer objective $f$) decreases with $K$ more quickly than generalization error.
The F1 score is already maximized at $K=5$.
These preliminary results indicate that in situations with limited memory, $K$-RMD for small $K$ (e.g. $K=5$) may be a reasonable fallback: it achieves results close to full backprop, and it runs about twice as fast.

\begin{table}
  \caption{\label{table:mnist} Hypercleaning metrics after 1000 hyperiters.}
  \center
  \vskip -1em
  \resizebox{0.8\columnwidth}{!}{\begin{tabular}{ccccc}
$K$ & Test Acc. & Val. Acc. & Val. Loss & F1\\
\hline
1 & 87.50 & 89.32 & 0.413 & 0.85 \\
5 & 88.05 & 89.90 & 0.383 & 0.89 \\
25 & 88.12 & 89.94 & 0.382 & 0.89 \\
50 & 88.17 & 90.18 & 0.381 & 0.89 \\
%75 & 88.28 & 90.22 & 0.381 & 0.89 \\
%95 & 88.30 & 90.24 & 0.380 & 0.88 \\
%99 & 88.33 & 90.24 & 0.380 & 0.88 \\
100 & 88.33 & 90.24 & 0.380 & 0.88 \\
\end{tabular}}
\end{table}

From a theoretical optimization perspective, we wonder whether $K$-RMD converges to a stationary point of $f$.
Data hypercleaning satisfies all of the assumptions of Theorem~\ref{th:convergence of K-step backprop} except that $B_t$ is not full column rank (since $M < N$).
In particular, the validation loss $f$ is deterministic and satisfies $\nabla_{\lambda} f = 0$.
Figure \ref{fig:mnist convergence} plots the norm of the true gradient $\d_\lambda f$ on a log scale at the $K$-RMD iterates for various $K$.
We see that, despite satisfying almost all assumptions, this problem exhibits biased convergence.
The limit of $\norm{\d_\lambda f}$ decreases slowly with $K$, but recall from Table \ref{table:mnist} that practical metrics improve more quickly.

\begin{figure}
  \center
  \includegraphics[width=0.6\columnwidth]{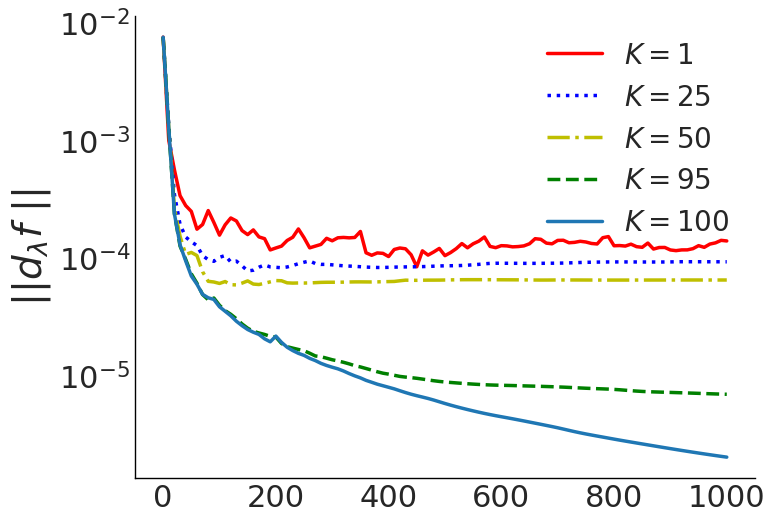}
  \caption{\label{fig:mnist convergence}
  $\norm{\d_\lambda f}$ vs. hyperiteration for hypercleaning.}
\end{figure}

%\paragraph{Results} Results are shown in the top half of Table~\ref{table:mnist}.
%s expected from the theory in Section \ref{sec:theory}, $K=1$ already does a very reasonable job.
%sing $K=25$ does slightly better, but beyond that, the results are not significantly improved.
%n fact, we see again that full RMD performs worse than moderate values of $K$.
%ence, using a short reverse depth can achieve accuracy comparable to full reverse depth, without incurring the associated memory and computational costs.
%his can also be seen in Figure~\ref{fig:mnist_rate}, where the learning curves for everything except $1$-RMD are virtually identical.

%begin{figure}
% \centering
% \includegraphics[trim={7ex 0 12ex 11ex}, clip, scale=0.28]{MNIST.pdf}
% \caption{\label{fig:mnist_rate}Upper-level objective loss vs. hyper-iteration for data hypercleaning experiment on MNIST.}%
%end{figure}
\vspace{-2mm}
\subsubsection{Task interaction}
\label{sec:cifar}
\vspace{-1mm}
We next consider the problem of multitask learning~\cite{evgeniou2005learning}.
Similar to~\citep{franceschi2017forward}, we formulate this as a hyperparameter optimization problem as follows.
The lower-level objective $g(w, \{C, \rho\})$ learns $V$ different linear models with parameter set $w = \{w_v\}_{v=1}^V$:
\begin{equation*}
  l(w) + \sum_{1 \leq i,j \leq K} C_{ij} \norm{w_i - w_j}^2 + \rho \sum_{v=1}^V \norm{w_v}^2
\end{equation*}
where $l(w)$ is the training loss of the multi-class linear logistic regression model, $\rho$ is a regularization constant, and $C$ is a nonnegative, symmetric hyperparameter matrix that encodes the similarity between each pair of tasks.
%More precisely, let $S=\{(x_i, y_i )\}_{i=1}^N$ be the training set, where $x_i \in \R^d$ is the feature vector and $y_i \in \{1, \dots, K\}$ is the label for the  $K$ different tasks.
After $100$ iterations of gradient descent with learning rate $0.1$, this yields $\hat{w}^*$.
%To ensure that $C$ is symmetric, and that $C_{ij}$ and $\rho$ are nonnegative, we re-parametrize them as $\rho = \text{softplus}(\nu)$ and $C = A+A^\top$, where $A_{ij} = \text{softplus}(B_{ij})$ and $B$ is an hyperparameter matrix.
%Thus, the hyperparameters to be optimized are $\lambda=\{B, \nu\}$.
The upper-level objective $c(\hat{w}^*)$ estimates the linear regression loss of the learned model $\hat{w}^*$ on a validation set.
Presumably, this will be improved by tuning $C$ to reflect the true similarities between the tasks.
The tasks that we consider are image recognition trained on very small subsets of the datasets CIFAR-$10$ and CIFAR-$100$.\footnote{
See Appendix~\ref{sec:cifar_appendix} for more details.}
%We extract features from the output of the average pooling layer in Resnet-$18$~\citep{he2016deep} which is trained on ImageNet~\citep{deng2009imagenet}.
%We use the same data pre-processing that is used for training Resnet architecture.

From an optimization standpoint, we are most interested in the upper-level loss on the validation set, since that is what is directly optimized, and its value is a good indication of the performance of the inexact gradient.
Figure~\ref{fig:cifar_rate} plots this learning curve along with two other metrics of theoretical interest: norm of the true gradient, and cosine similarity between the true and approximate gradients.
%as we train the models for $20000$ iterations. In addition to the objective value, we show the norm of the exact hypergradients and 
%the cosine similarity between the exact hypergradient and the approximate hypergradient given by $K$-RMD.
%The norm of the gradient measures how the hyperparameters converge to stationary points. 
%
In CIFAR100, the validation error and gradient norm plots show that $K$-RMD converges to an approximate stationary point with 
a bias that rapidly decreases as $K$ increases, agreeing with Proposition~\ref{pr:exp convergence}.
Also, we find that negative values exist in the cosine similarity of $1$-RMD, which implies 
that not all the assumptions in Theorem~\ref{th:convergence of K-step backprop} hold for this 
problem (e.g. $B_t$ might not be full rank, or the the inner problem might not be locally strong convex around $\hat{w}^*$.)
%This could explain the non-convergence of the gradients generated by the $K$-RMD algorithms, although they might converge to exact stationary points after more iterations.
In CIFAR10, some unusal behavior happens. For $K > 1$, the truncated gradient and the full gradient directions eventually become almost the same.
%only happen if the truncation error norm becomes negligible compared to the true gradient norm or both vector have the same direction.
We believe this is a very interesting observation but beyond the scope of the paper to explain.

%From a learning standpoint, we are interested in the generalization to unseen data, so we also report the best testing accuracy of each method.
In Table~\ref{table:cifar}, we report the testing accuracy over 10 trials.
While in general increasing the number of back-propagation steps improves accuracy, the gaps are small.
%One reason could be that the peak in the accuracy happens very early in the optimization (see the average number of the iterations) when none of the methods have converged yet.
%This indicates that convergence to close proximity of stationary points does not necessarily mean better generalization.
A thorough investigation of the relationship between convergence and generalization is an interesting open question of both theoretical and practical importance.
\begin{table}[h]
\caption{\label{table:cifar}Test accuracy for task interaction. Few-step $K$-RMD achieves similar performance as full RMD.}
\centering
\resizebox{.8\columnwidth}{!}{\begin{tabular}{clcccc}
&{\bf Method} & {\bf Avg. Acc.} & {\bf Avg. Iter.} & {\bf Sec/iter.}\\
\hline
\multirow{4}{*}{\begin{turn}{-90}\hspace{-8pt}\small{CIFAR-10}\end{turn}}
&$1$-RMD & $61.11\pm1.23$ & $3300$ & $0.8$\\
&$5$-RMD & $61.33\pm1.08$ & $4950$ & $1.3$\\
&$25$-RMD & $61.31\pm1.24$ & $4825$ & $1.4$\\
&Full RMD & $61.28\pm1.21$ & $4500$ & $2.2$\\
\hline
\multirow{4}{*}{\begin{turn}{-90}\hspace{-8pt}\small{CIFAR-100}\end{turn}}
&$1$-RMD& $34.37\pm0.63$ & $7440$ & $1.0$\\
&$5$-RMD& $34.34\pm0.68$ & $8805$ & $1.4$\\
&$25$-RMD& $34.51\pm0.69$ & $8660$ & $1.6$\\
&Full RMD& $34.70\pm0.64$ & $5670$ & $2.8$
\end{tabular}}
\end{table}

\begin{figure}[t]
  \centering
  \subfigure{\includegraphics[scale=0.13]{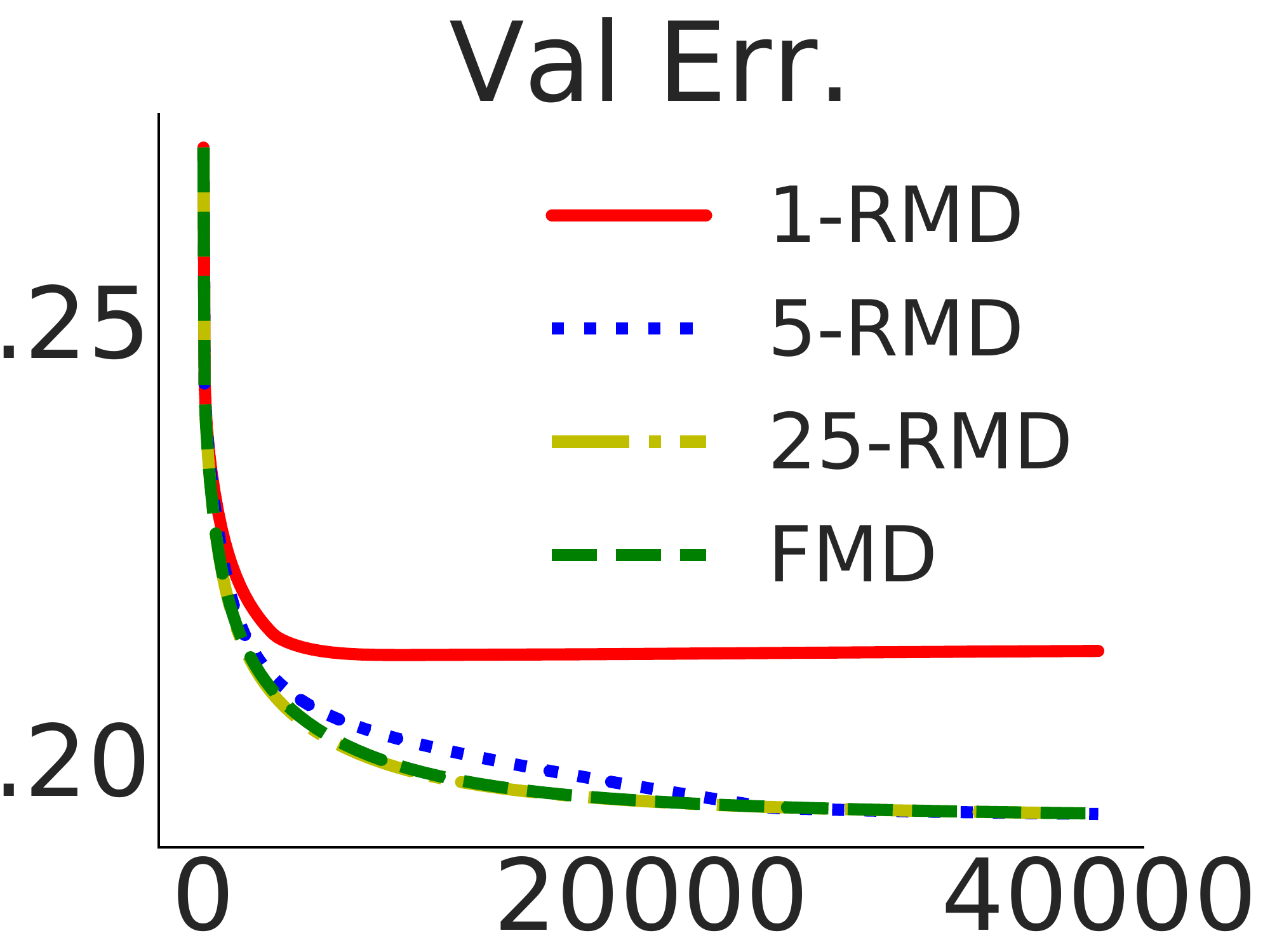}}
  \subfigure{\includegraphics[scale=0.13]{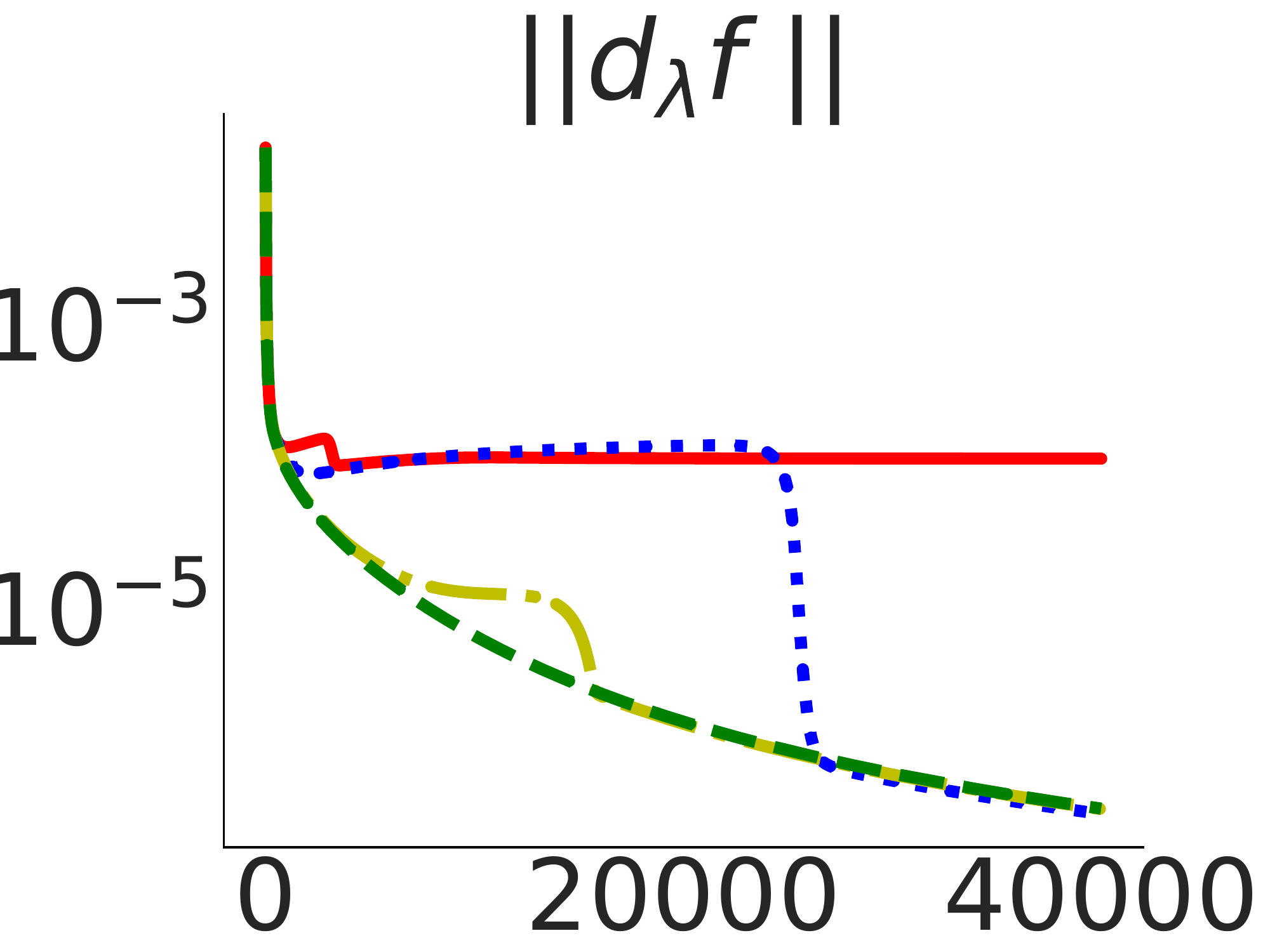}}
  \subfigure{\includegraphics[scale=0.13]{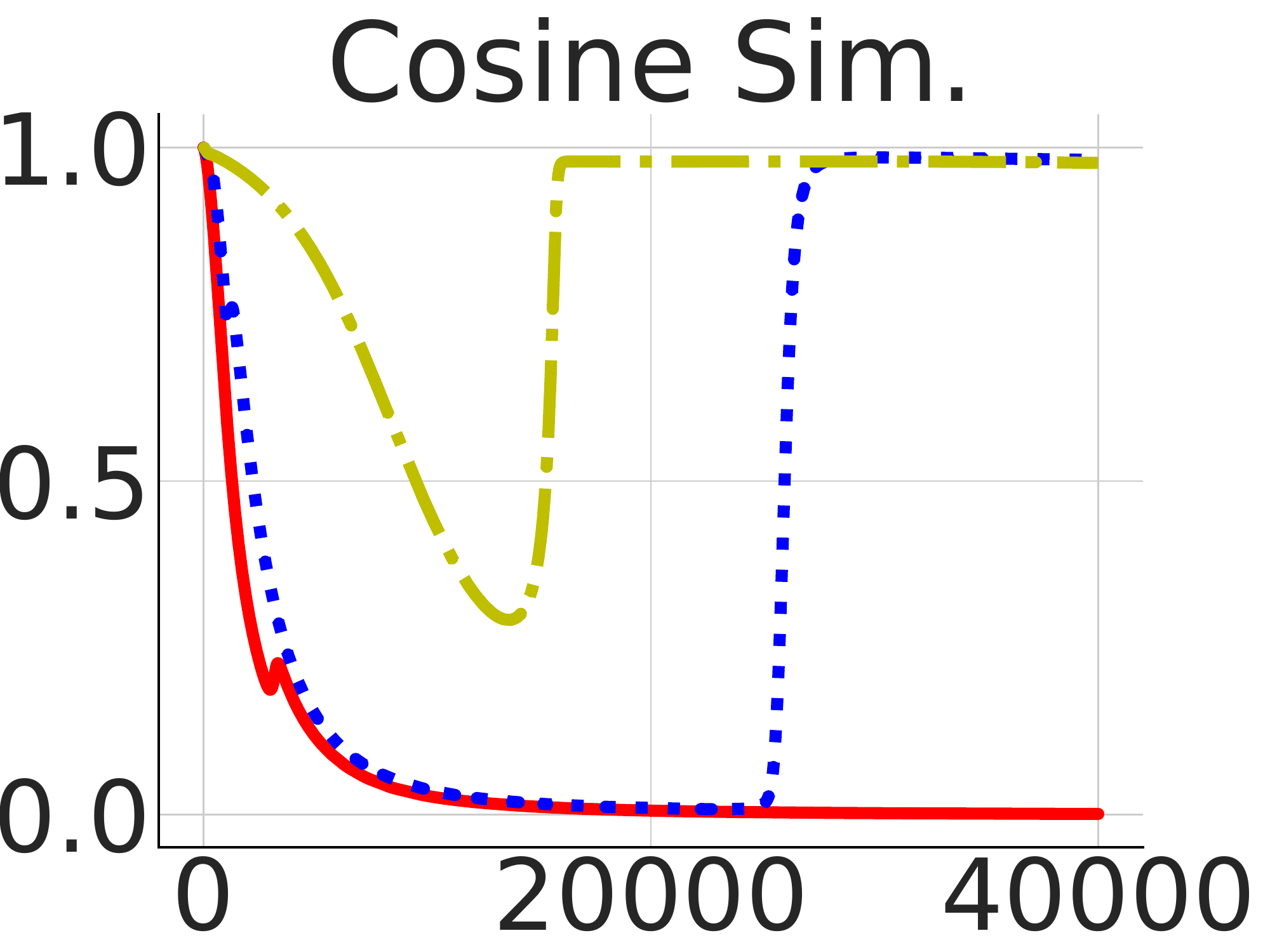}}\\
  %CIFAR100
  \subfigure{\includegraphics[scale=0.13]{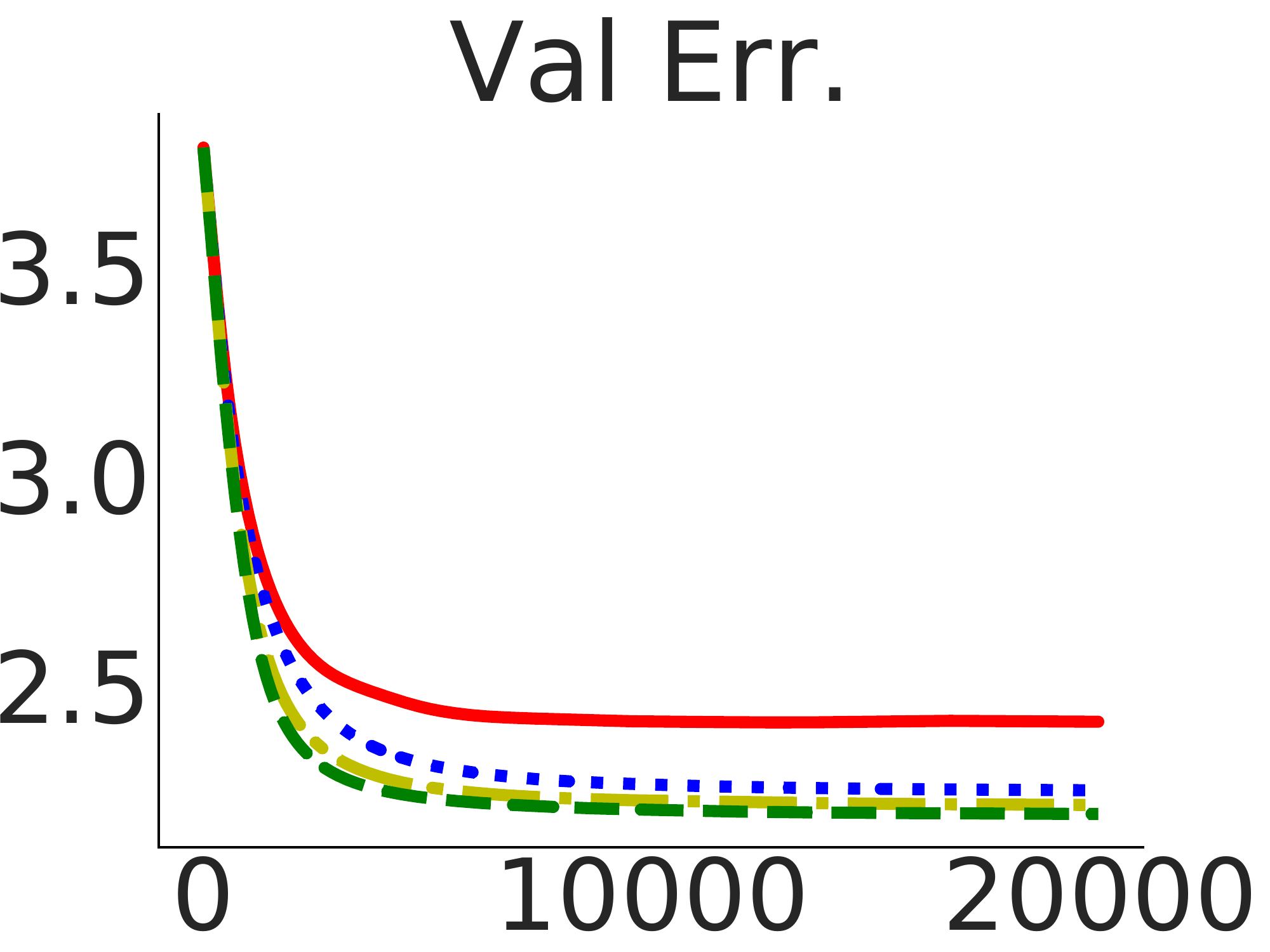}}
  \subfigure{\includegraphics[scale=0.13]{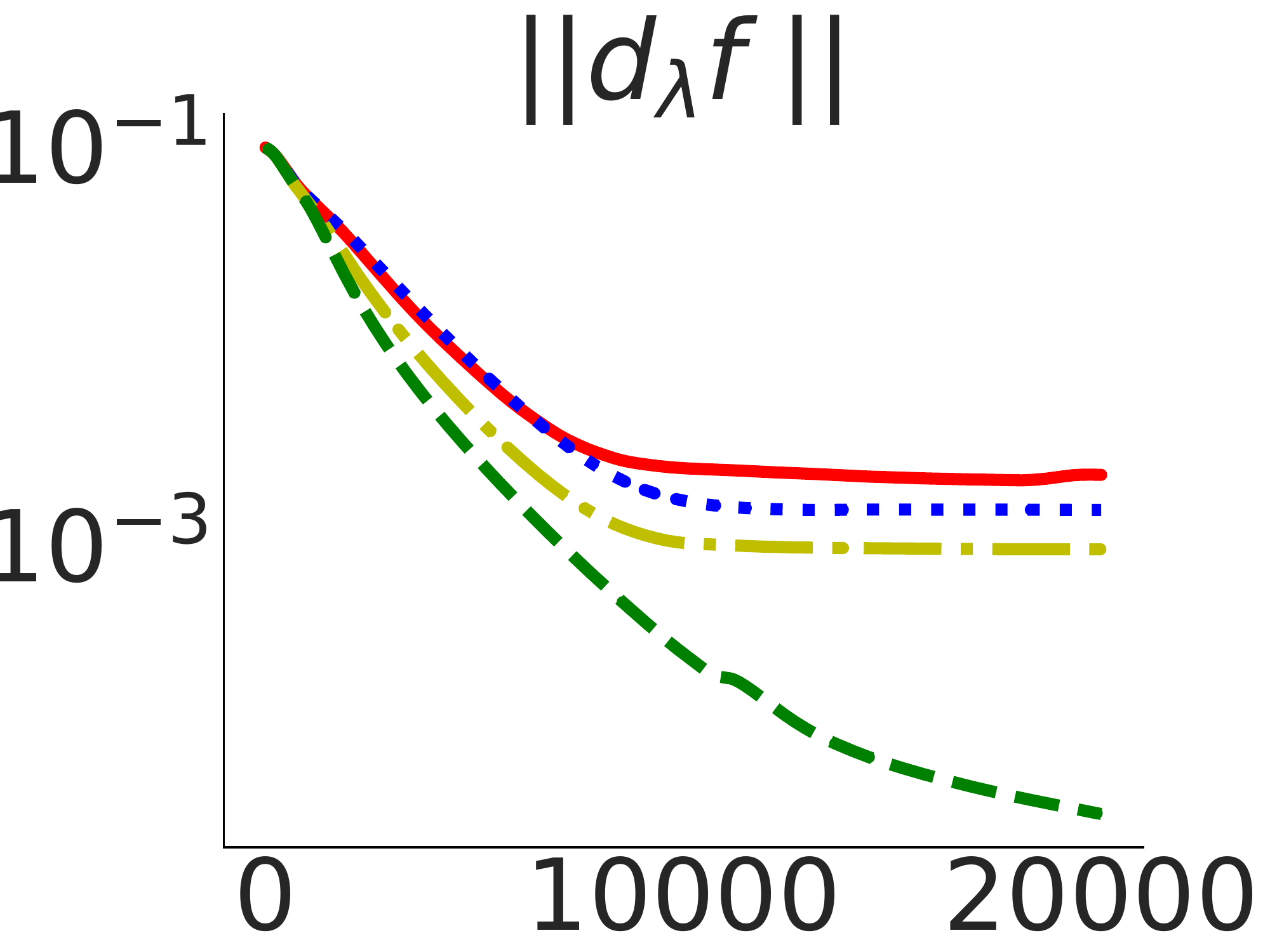}}
  \subfigure{\includegraphics[scale=0.13]{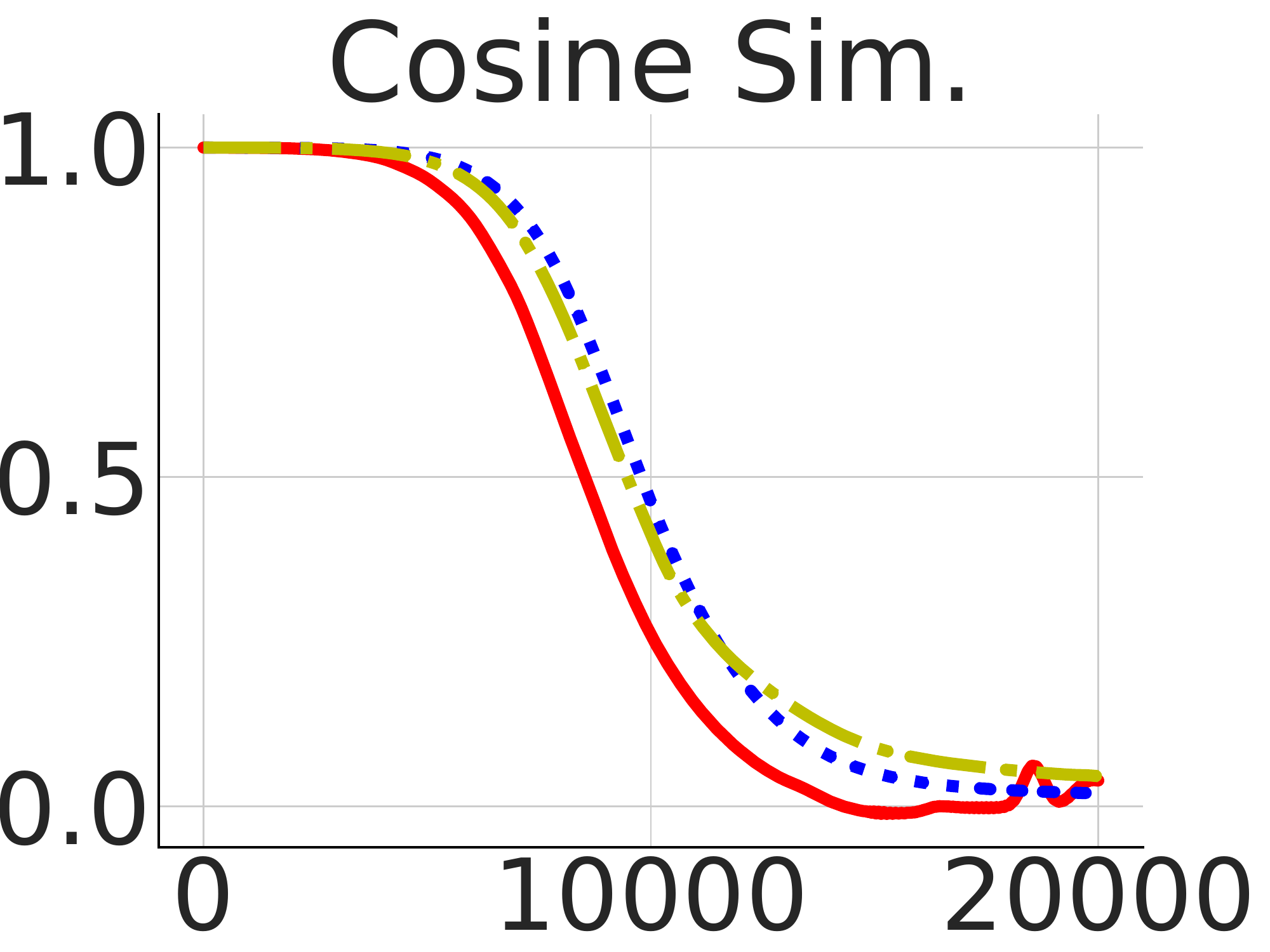}}
  %\subfigure{\includegraphics[trim={7ex 0 13ex 0ex}, clip,scale=0.21]{CIFAR100_acc}}
  \caption{\small{\label{fig:cifar_rate}Upper-level objective loss (first column), norm of the exact gradient (second column), and cosine similarity (last column) vs. hyper-iteration on CIFAR10 (first row) and CIFAR100 (second row) datasets.}}
\end{figure}

\vspace{-2mm}
\subsection{Meta-learning: One-shot classification} \label{sec:omniglot}
\vspace{-1mm}
The aim of this experiment is to evaluate the performance of truncated back-propagation in multi-task, stochastic optimization problems.
We consider in particular the one-shot classification problem~\citep{finn2017model}, where each task $\TT$ is a $k$-way classification problem and the goal is learn a hyperparameter $\lambda$ such that each task can be solved with few training samples. 

In each hyper-iteration, we sample a task, a training set, and a validation set as follows: 
First, $k$ classes are randomly chosen from a pool of classes % in a meta-dataset $\DD$ 
to define the sampled task $\TT$. Then the training set $S=\{(x_i, y_i)\}_{i=1}^k$ is created 
by randomly drawing one training example $(x_i, y_i)$ from each of the $k$ classes. The validation set $Q$ is constructed similarly, but with more examples from each class.
The lower-level objective $g_S(w, \lambda)$ is
\begin{equation*} \textstyle
\sum_{(x_i, y_i) \in S} l(nn(x_i; w, \lambda), y_i) + \sum_{j=1}^V \rho_j ||w_j - c_j||^2
\end{equation*}
where $l(\cdot, \cdot)$ is the $k$-way cross-entropy loss, and $nn(\cdot; w,\lambda)$ is a deep neural network parametrized by $w=\{w_1, \dots, w_V\}$ and optionally hyperparameter  $\lambda$. To prevent overfitting in the lower-level optimization,
we regularize each parameter $w_j$ to be close to center $c_j$ with weight $\rho_j>0$.
Both $c_j$ and $\rho_j$ are hyperparameters, as well as the inner learning rate $\gamma$.
The upper-level objective is the loss of the trained network on the sampled validation set $Q$.
%which is written as $\sum_{( x_i, y_i) \in Q} l(nn(x_i; \hat{w}^*, \lambda), y_i)$ and is used to construct the first-order stochastic oracle  of $\nabla_\lambda f_S$ and $\nabla_{\hat{w}^*} f_S$ mentioned in Section~\ref{sec:setup}.
In contrast to other experiments, this is a stochastic optimization problem.
%in which we evaluate the performance of $K$-RMD on new unseen tasks.
Also, $\AA_{\lambda}(S)(x_i) = nn( x_i ; \hat{w}^*, \lambda)$ depends directly
on the hyperparameter $\lambda$, in addition to the indirect dependence through $\hat{w}^*$ (i.e. $\nabla_\lambda f  \neq 0$).

We use the Omniglot dataset \citep{lake2015human} and a similar neural network as used in~\cite{finn2017model} with small modifications.
Please refer to Appendix~\ref{app:one_shot} for more details about the model and the data splits.
We set $T= 50$ and optimize over the hyperparameter $\lambda = \{\lambda_{l_1}, \lambda_{l_2}, c, \rho, \gamma\}$.
The average accuracy of each model is evaluated over $120$ randomly sampled training and validation sets from the meta-testing dataset.
 %$(S, Q) \sim \DD_{test}$.
For comparison, we also try using full RMD with a very short horizon $T=1$, which is common in recent work on few-shot learning~\cite{finn2017model}.

The statistics are shown in Table~\ref{table:oneshot} and the learning curves in Figure~\ref{fig:omniglot_rate}.
 In addition to saving memory, all truncated methods are faster than full RMD, sometimes even five times faster. %Furthermore, $10$- and $15$-RMD perform \emph{better} than full RMD.
These results suggest that running few-step back-propagation with more hyper-iterations can be more efficient 
than the full RMD. To support this hypothesis, we also ran $1$-RMD and $10$-RMD for an especially large number of hyper-iterations ($15$k).
Even with this many hyper-iterations, the total runtime is less than full RMD with $5000$ iterations, and the results are significantly improved.
We also find that while using a short horizon ($T=1$) is faster, it achieves
a lower accuracy at the same number of iterations.

Finally, we verify some of our theorems in practice.
Figure~\ref{fig:omniglot_rate} (fourth plot) shows that when the lower-level problem is regularized, the relative $\ell_2$ error between the $K$-RMD approximate gradient and the exact gradient decays exponentially as $K$ increases.
This was guaranteed by Proposition~\ref{pr:exp convergence}.
However, this exponential decay is not seen for the non-regularized model ($\rho = 0$).
This suggests that the local strong convexity assumption is essential in order to have exponential decay in practice.
Figure~\ref{fig:omniglot_rate} (third plot) shows the cosine similarity between the inexact gradient and full gradient over the course of meta-training.
Note that the cosine similarity measures are always positive, indicating that the inexact gradients are indeed descent directions.
It also seems that the cosine similarities show a slight decay over time.
%This might be attributed to the fact that we are using a fixed horizon $T$ to solve the lower-level problem;  as the properties of the lower-level problem changes over course of hyper-iterations, different horizons might be required to maintain the same cosine similarity.

\begin{table}
\caption{\label{table:oneshot}Results for one-shot learning on Omniglot dataset. $K$-RMD reaches similar performance as full RMD, is considerably faster, and requires less memory.}%
\center
\begin{tabular}{lccc}
{\bf Method} & {\bf Accuracy} & {\bf iter.} & {\bf Sec/iter.}\\
\hline
$1$-RMD & $95.6$ & $5$K & $0.4$\\
$10$-RMD &$96.3$ & $5$K & $0.7$\\
$25$-RMD & $96.1$ & $5$K &$1.3$\\
Full RMD & $95.8$ & $5$K & $2.2$\\
\hline
\hline
$1$-RMD & $97.7$ & $15$K & $0.4$\\
$10$-RMD & $97.8$ & $15$K & $0.7$\\
Short horizon & $96.6$ & $15$K & $0.1$\\
\end{tabular}
\end{table}

\begin{figure}[t]
  \centering
  \subfigure{\includegraphics[trim={7ex 0 12ex 0}, clip, scale=0.2]{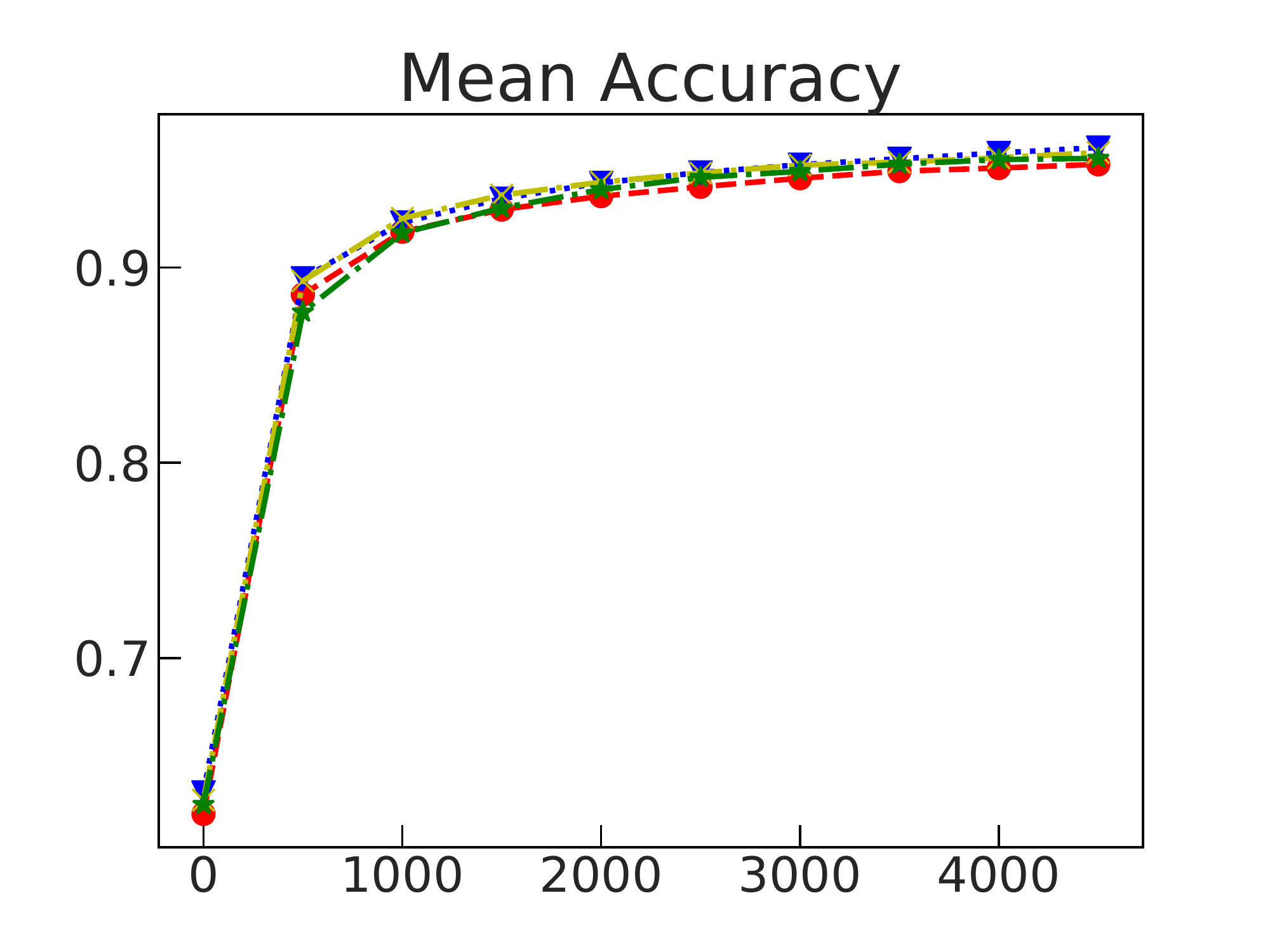}}
  \subfigure{\includegraphics[trim={7ex 0 12ex 0}, clip, scale=0.2]{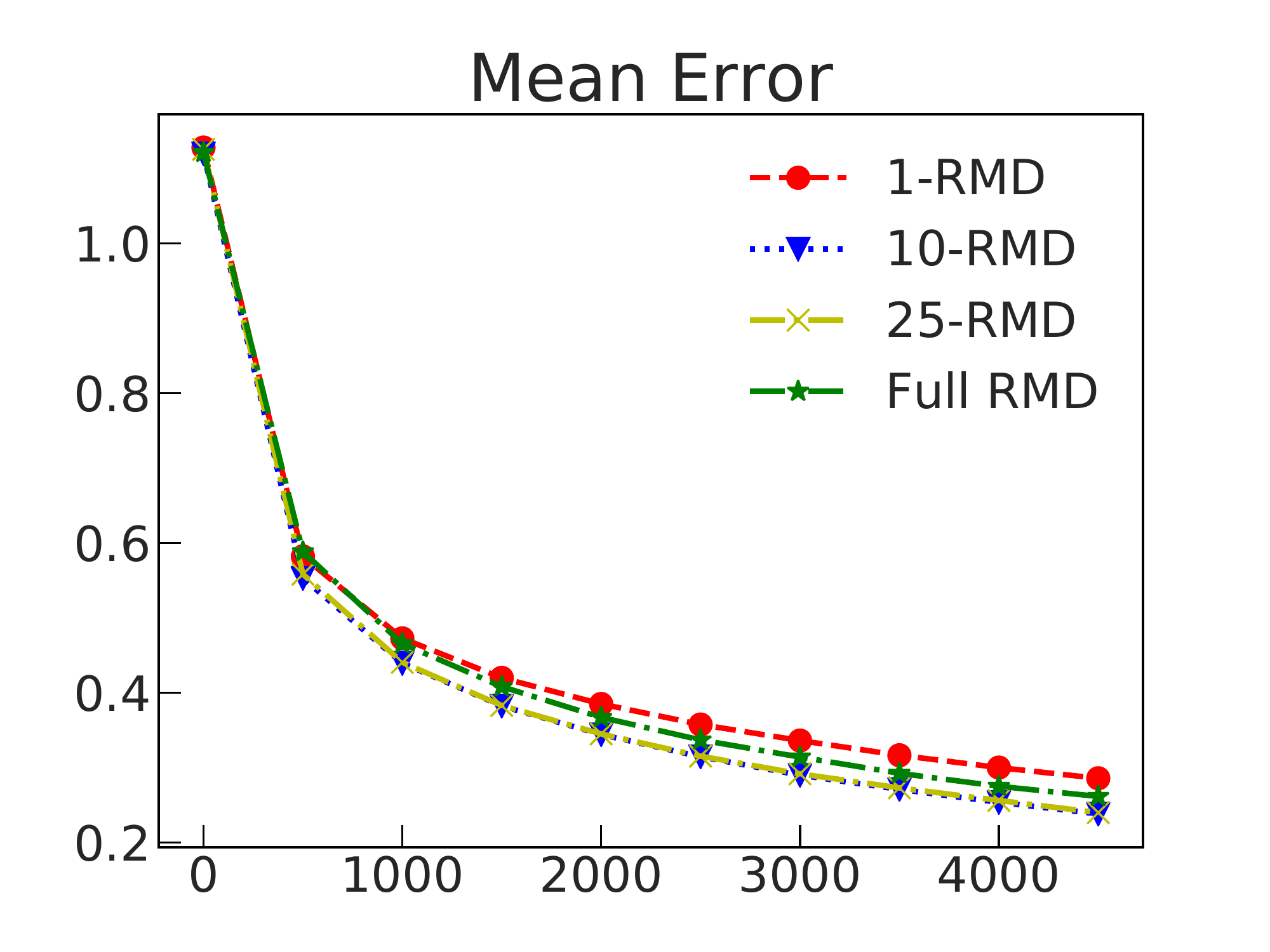}}\\\vspace{-5mm}
  \subfigure{\includegraphics[trim={7ex 0 12ex 0}, clip, scale=0.2]{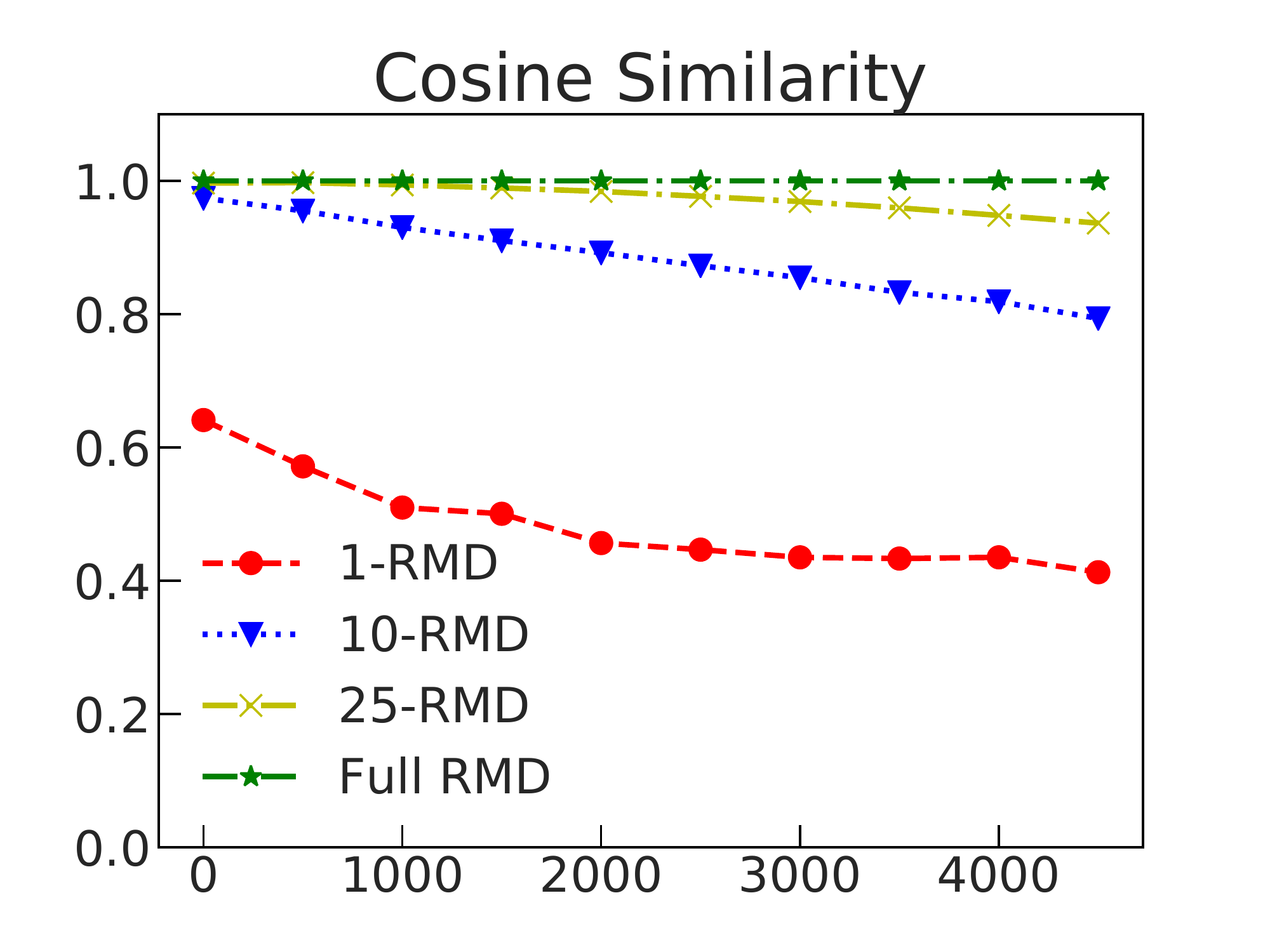}}
  \subfigure{\includegraphics[trim={7ex 0 12ex 0}, clip, scale=0.2]{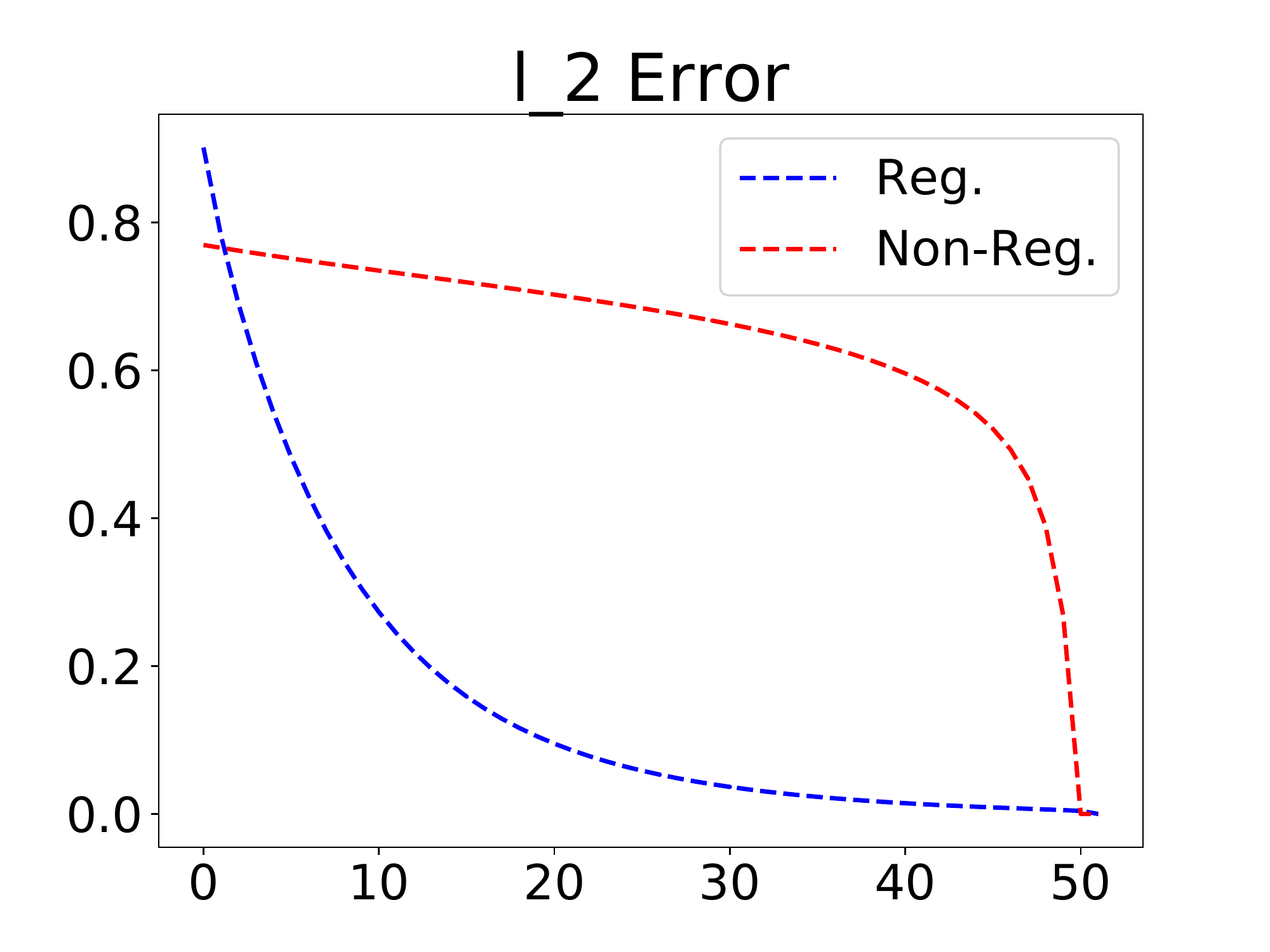}}
  \caption{\small{\label{fig:omniglot_rate}Omniglot results. \textbf{Plots 1 and 2:} Test accuracy and val. error vs. number of hyper-iterations for different RMD depths. $K$-RMD methods show similar performance as the full RMD. \textbf{Plot 3:} Cosine similarity between inexact gradient and full RMD over hyper-iterations. \textbf{Plot 4:} Relative $\ell_2$ error of inexact gradient and full RMD vs. reverse depth. Regularized version shows exponential decay.}}
  \vspace{-5mm}
\end{figure}

%\begin{figure}[t]
%  \centering
%  \subfigure{\includegraphics[scale=0.17]{cos_shot}%}
%  \subfigure{\includegraphics[scale=0.17]{l2_shot}}
%  \caption{\label{fig:omniglot_grad_err}{\bf Left}: Cosine similarity between inexact gradient and full RMD over hyper-iterations. {\bf Right}: Relative $\ell_2$ error of inexact gradient and full RMD vs. reverse depth. Regularized version shows exponential decay.}
%\end{figure}

\vspace{-2mm}
\section{CONCLUSION}
\vspace{-1mm}
We analyze $K$-RMD, a first-order heuristic for solving bilevel optimization problems when the lower-level optimization is itself approximated in an iterative way.
We show that $K$-RMD is a valid alternative to full RMD from both theoretical and empirical standpoints.
Theoretically, we identify sufficient conditions for which the hyperparameters converge to an approximate or exact stationary point of the upper-level objective.
The key observation is that when $\hat w^*$ is near a strict local minimum of the lower-level objective, gradient approximation error decays exponentially with reverse depth.
Empirically, we explore the properties of this optimization method with four proof-of-concept experiments.
We find that although exact convergence appears to be uncommon in practice, the performance of $K$-RMD is close to full RMD in terms of application-specific metrics (such as generalization error).
It is also roughly twice as fast.
These results suggest that in hyperparameter optimization or meta learning applications with memory constraints, truncated back-propagation is a reasonable choice.

Our experiments use a modest number of parameters $M$, hyperparameters $N$, and horizon length $T$.
This is because we need to be able to calculate both $K$-RMD and full RMD in order to compare their performance.
One promising direction for future research is to use $K$-RMD for bilevel optimization problems that require powerful function approximators at both levels of optimization.
Truncated RMD makes this approach feasible and enables comparing bilevel optimization to other meta-learning methods on difficult benchmarks.

\bibliographystyle{apalike}
\bibliography{ref}

\clearpage

\onecolumn

\appendix
\section*{Appendix}

%\section{PROOFS} %\label{app:proofs}

\section{Proof of Proposition~\ref{pr:exp convergence}} \label{app:proof of exp convergence}
%\prExpConvergence
\expConvergence*

\begin{proof}
	Let $\d_{\lambda} f  - h_{T-K}  = e_K$. By definition of $h_{T-K}$, 
	\begin{align*}
	e_K =  \left( \sum_{t=0}^{T-K}  B_{t} A_{t+1} \cdots A_{T-K}\right) A_{T-K+1} \cdots A_{T}  \nabla_{\hat{w}^*} f 
	\end{align*}
	Therefore, when $g$ is locally $\alpha$-strongly convex with respect to $w$ in the neighborhood of $\{w_{T-K-1}, \dots, w_T\}$,
	\begin{align*}
	\norm{e_K} &\leq \norm{ \sum_{t=0}^{T-K}  B_{t} A_{t+1} \cdots A_{T-K}} \norm{ A_{T-K+1} \cdots A_{T}  \nabla_{\hat{w}^*} f} \\
	&\leq (1 - \gamma \alpha)^{K} \norm{ \nabla_{\hat{w}^*} f}   \norm{ \sum_{t=0}^{T-K}  B_{t} A_{t+1} \cdots A_{T-K}} 
	\end{align*}
	Suppose $g$ is $\beta$-smooth but nonconvex. In the worst case, if the smallest eigenvalue of $\nabla_{w,w} g(w_{t-1}, \lambda)$ is $-\beta$, then $\norm{A_t} = 1 + \gamma \beta \leq 2$ for $t = 0,\dots,T-K$. This gives the bound in~\eqref{eq:nonconvex bound}. However, if $g$ is globally strongly convex, then
	\begin{align*}
	\norm{e_K} &\leq \norm{ \nabla_{\hat{w}^*} f} (1 - \gamma \alpha)^{K}  \max_{t\in\{0,\dots,T-K\} } \norm{B_t} \sum_{t=0}^{T-K} (1 - \gamma \alpha)^{t} 
	\end{align*}
	The bound~\eqref{eq:convex bound} uses the fact that
	$\sum_{t=0}^{T-K} (1 - \gamma \alpha)^{t}  \leq  \sum_{t=0}^{\infty} (1 - \gamma \alpha)^{t}  = \frac{1}{\gamma \alpha}  $
\end{proof}

\section{Proof of Lemma~\ref{lm:bound of many terms}} \label{app:proof of bound of many terms}
%	h_{T-1}^\top \d_{\lambda} f &\geq  \norm{\nabla_{\hat{w}^*} f}^2  c

\boundOfManyTerms*

\begin{proof}
  To illustrate the idea, here we prove the case where $K=1$. For $K>1$, similar steps can be applied. To prove the statement, we first expand the inner product by definition
\begin{align*}
h_{T-1}^\top \d_{\lambda} f = \norm{h_{T-1}}^2 + \left(  B_T  \nabla_{\hat{w}^*} f\right)^\top \left(  \sum_{t=0}^{T-1}  B_{t} A_{t+1}  \cdots A_{T-1} \right)  A_{T}  \nabla_{\hat{w}^*} f
\end{align*}
where we recall $h_{T-1} = B_T  \nabla_{\hat{w}^*} f$ as $\nabla_\lambda f = 0 $ by assumption.

Next we show a technical lemma, which provides a critical tool to bound the second term above; its proof is given in the next section.% Appendix~\ref{app:proof of bound of single term}.
\begin{lemma} \label{lm:bound of single term}
Let $g$ be $\alpha$-strongly convex and $\beta$-smooth. Assume $B_t$ and $A_t$ are Lipschitz continuous in $w$, and assume $B_T$ has full column rank. For $\gamma \leq \frac{1}{\beta}$, 
\begin{align*}
&\left(  B_T  \nabla_{\hat{w}^*} f\right)^\top   B_{t} A_{t+1} \cdots A_{T}  \nabla_{\hat{w}^*} f  \\
&\geq  (1-\gamma\alpha)^{T-t} \norm{ B_T  \nabla_{\hat{w}^*} f}^2   - \norm{ \nabla_{\hat{w}^*} f  }^2 O\left( \frac{e^{-\alpha\gamma(T-1)}}{1- e^{-\alpha\gamma}} +   \left(  \gamma (\beta-\alpha)  \right)^{T-t}  \right)
\end{align*}
\end{lemma}

By Lemma~\ref{lm:bound of single term},  we can then write
\begin{align*}
h_{T-1}^\top \d_{\lambda} f  
\geq  \norm{ B_T  \nabla_{\hat{w}^*} f}^2 \left( 1+  \sum_{t=0}^{T-1} (1-\gamma\alpha)^{T-t} \right)
 -   \norm{\nabla_{\hat{w}^*} f}^2  O \left(  \sum_{t=0}^{T-1}  \frac{e^{-\alpha\gamma(T-1)}}{1- e^{-\alpha\gamma}} + \left(  \gamma (\beta-\alpha)  \right)^{T-t} \right)  
\end{align*}
Because 
\begin{align*}
\sum_{t=0}^{T-1}\left(  \gamma (\beta-\alpha)  \right)^{T-t} 
&= \sum_{k=1}^{T}\left(  \gamma (\beta-\alpha)  \right)^k \leq  \frac{\gamma (\beta-\alpha)}{1-\gamma (\beta-\alpha)}
 &\text{($\because \gamma \leq \beta$)}
\end{align*}
and  $B_T^\top B_T$ is non-singular by assumption, 
\begin{align*}
h_{T-1}^\top \d_{\lambda} f  
&\geq  \norm{ \nabla_{\hat{w}^*} f}^2 \Omega(1)
-   \norm{\nabla_{\hat{w}^*} f}^2  O \left(   \frac{ T e^{-\alpha\gamma(T-1)}}{1- e^{-\alpha\gamma}}  +  \frac{\gamma (\beta-\alpha)}{1-\gamma (\beta-\alpha)} \right)\\
&\geq C  \norm{ \nabla_{\hat{w}^*} f}^2 
\end{align*}
for some $c > 0$, when $T$ is large enough and $\gamma$ is small enough.
The implication holds because $\norm{\d_\lambda f } \leq  O(\norm{ \nabla_{\hat{w}^*} f})$.
\end{proof}

\subsection{Proof of Lemma~\ref{lm:bound of single term}} \label{app:proof of bound of single term}
\begin{proof}
Let  $C_A$ and $C_B$ be the Lipschitz constant of $A_t$ and $B_t$. 
First, we see that the  inner product can be lower bounded by the following terms
\begin{align*}
\left(  B_T  \nabla_{\hat{w}^*} f\right)^\top   B_{t} A_{t+1} \cdots A_{T}  \nabla_{\hat{w}^*} f  
\geq  (1-\gamma\alpha)^{T-t} \norm{ B_T  \nabla_{\hat{w}^*} f}^2   - \Delta_1 - \Delta_2 - \Delta_3 
\end{align*}
where
\begin{align*}
\Delta_1 = C_B \norm{B_T  \nabla_{\hat{w}^*} f}  \norm{ \nabla_{\hat{w}^*} f }  \norm{w_{T-1} - w_{t-1}} \norm{ A_{t+1} \cdots A_{T}}
\end{align*}
\begin{align*}
\Delta_2 =  C_A  \norm{ B_T^\top   B_{T}  \nabla_{\hat{w}^*} f} \norm{ \nabla_{\hat{w}^*} f} \sum_{k=t+1}^{T-1}\norm{w_{T-1} - w_{k-1}} \norm{ A_{t+1} \cdots A_{k-1}} \norm{ A_{T} }^{T-k}
\end{align*}
\begin{align*}
\Delta_3 =  \norm{\nabla_{\hat{w}^*} f} \norm{B_T^\top B_T \nabla_{\hat{w}^*} f } \norm{A_k - (1 - \gamma \alpha) I}^{T-k}
\end{align*}
The above lower bounds can be shown by the following inequalities: 
\begin{align*}
&\left(  B_T  \nabla_{\hat{w}^*} f\right)^\top   B_{t} A_{t+1} \cdots A_{T}  \nabla_{\hat{w}^*} f  \\
&\geq
\nabla_{\hat{w}^*} f^\top \left( B_T^\top   B_{T} \right) A_{t+1} \cdots A_{T} 
\nabla_{\hat{w}^*} f   - C_B \norm{B_T  \nabla_{\hat{w}^*} f}\norm{w_{T-1} - w_{t-1}} \norm{ A_{t+1} \cdots A_{T}  \nabla_{\hat{w}^*} f } 
\end{align*}
\begin{align*}
&\nabla_{\hat{w}^*} f^\top \left( B_T^\top   B_{T} \right) A_{t+1} \cdots A_{T} \nabla_{\hat{w}^*} f \\
&\geq
\nabla_{\hat{w}^*} f^\top  \left( B_T^\top   B_{T} \right) A_{t+1} \cdots A_{T-2} A^2_{T} \nabla_{\hat{w}^*} f 
  - C_A \norm{w_{T-1} - w_{T-2}} \norm{ A_{t+1} \cdots A_{T-2}} \norm{  A_{T}} \norm{  B_T^\top   B_{T} \nabla_{\hat{w}^*} f}\norm{ \nabla_{\hat{w}^*} f } \\
&\geq
\nabla_{\hat{w}^*} f^\top  B_T^\top B_t A_T^{T-t} \nabla_{\hat{w}^*} f
 - C_A  \norm{ B_T^\top   B_{T}  \nabla_{\hat{w}^*} f} \norm{ \nabla_{\hat{w}^*} f} \sum_{k=t+1}^{T-1}\norm{w_{T-1} - w_{k-1}} \norm{ A_{t+1} \cdots A_{k-1}} \norm{ A_{T} }^{T-k}
\end{align*}
\begin{align*}
\nabla_{\hat{w}^*} f^\top  B_T^\top B_T A_T^{T-t} \nabla_{\hat{w}^*} f 
\geq   (1-\gamma\alpha)^{T-t} \nabla_{\hat{w}^*} f^\top  B_T^\top B_T  \nabla_{\hat{w}^*} f - \norm{\nabla_{\hat{w}^*} f} \norm{B_T^\top B_T \nabla_{\hat{w}^*} f } \norm{A_T - (1 - \gamma \alpha) I}^{T-t}
\end{align*}

Next we upper bound the error terms: $\Delta_1$, $\Delta_2$, and $\Delta_3$. We will use the fact that gradient descent converges linearly when optimizing a strongly convex and smooth function~\citep{hazan2016introduction}.
\begin{lemma} \label{lm:gradient descent}
Let $w_0$ be the initial condition. Running gradient descent to optimize an $\alpha$-strongly convex and $\beta$-smooth function $g$, with step size $ 0 < \gamma \leq \frac{1}{\beta}$, generates a sequence $\{w_t\}$ satisfying
	\begin{align}
	\norm{w_t - w^*} \leq D e^{-\alpha\gamma t}
	\end{align}
	where $D = \norm{w_0 - w^*}$ and $w^* = \argmin g(w)$.
\end{lemma}
Lemma~\ref{lm:gradient descent} implies for $T\geq t$, $
\norm{w_{T} - w_{t}} \leq  2 De^{-\alpha\gamma t} $. 

Now we proceed to bound the errors  $\Delta_1$, $\Delta_2$, and $\Delta_3$.
\paragraph{Bound on $\Delta_1$}
Because
\begin{align*}
\norm{w_{T-1} - w_{t-1}} \norm{ A_{t+1} \cdots A_{T}} 
&\leq
2 D  e^{- \alpha\gamma (t-1)}   (1-\gamma \alpha)^{T-t}\\
&\leq 2 D  e^{- \alpha\gamma (t-1)} e^{-\gamma \alpha(T-t)} \\
&= 2 D  e^{ - \alpha\gamma\left ( T-1 \right) } 
\end{align*}
we can upper bound $\Delta_1$ by
\begin{align*}
\Delta_1 &= C_B \norm{B_T  \nabla_{\hat{w}^*} f}  \norm{ \nabla_{\hat{w}^*} f }  \norm{w_{T-1} - w_{t-1}} \norm{ A_{t+1} \cdots A_{T}} \\
& \leq  \norm{B_T \nabla_{\hat{w}^*} f} \norm{\nabla_{\hat{w}^*} f} \times 2 C_B D  e^{ - \alpha\gamma\left ( T-1 \right) } 
\end{align*}

\paragraph{Bound on $\Delta_2$}
Because
\begin{align*}
\sum_{k=t+1}^{T-1}\norm{w_{T-1} - w_{k-1}} \norm{ A_{t+1} \cdots A_{k-1}} \norm{ A_{T} }^{T-k}
&\leq \sum_{k=t+1}^{T-1} 2De^{-\alpha\gamma(k-1)} (1-\alpha\gamma)^{k-1-t+T-k} \\
&\leq 2D  (1-\alpha\gamma)^{T-t-1} \sum_{k=t+1}^{T-1} e^{-\alpha\gamma(k-1)}  \\
&\leq 2D  (1-\alpha\gamma)^{T-t-1} e^{-\alpha\gamma t} \sum_{k=t+1}^{T-1} e^{-\alpha\gamma(k-t-1)}  \\
&\leq 2D   e^{-\alpha\gamma(T-1)} \sum_{m=0}^{T-t} e^{-\alpha\gamma m}  \\
&\leq  \frac{ 2D }{1- e^{-\alpha\gamma}}   e^{-\alpha\gamma(T-1)}   
\end{align*}
we can upper bound $\Delta_2$ by
\begin{align*}
\Delta_2 &=  C_A  \norm{ B_T^\top   B_{T}  \nabla_{\hat{w}^*} f} \norm{ \nabla_{\hat{w}^*} f} \sum_{k=t+1}^{T-1}\norm{w_{T-1} - w_{k-1}} \norm{ A_{t+1} \cdots A_{k-1}} \norm{ A_{T} }^{T-k}
\\
&=   \norm{ B_T^\top   B_{T}  \nabla_{\hat{w}^*} f} \norm{ \nabla_{\hat{w}^*} f} \times  \frac{ 2 C_A D }{1- e^{-\alpha\gamma}}   e^{-\alpha\gamma(T-1)} 
\end{align*}

\paragraph{Bound on $\Delta_3$}
Because 
\begin{align*}
\norm{A_k - (1 - \gamma \alpha) I}  = \norm{ \gamma \left( \alpha I -  \nabla_{w}^2 f(w_{k-1}) \right) } \leq \gamma (\beta-\alpha) 
\end{align*}
we can upper bound $\Delta_3$ by
\begin{align*}
\Delta_3 =  \norm{\nabla_{\hat{w}^*} f} \norm{B_T^\top B_T \nabla_{\hat{w}^*} f } \norm{A_t - (1 - \gamma \alpha) I}^{T-t}
&\leq  \norm{\nabla_{\hat{w}^*} f} \norm{B_T^\top B_T \nabla_{\hat{w}^*} f } \left(  \gamma (\beta-\alpha)  \right)^{T-t}
\end{align*}

\paragraph{Final Result}

Using the bounds on  $\Delta_1$, $\Delta_2$, and $\Delta_3$, we prove the final result.
\begin{align*}
&\left(  B_T  \nabla_{\hat{w}^*} f\right)^\top   B_{t} A_{t+1} \cdots A_{T}  \nabla_{\hat{w}^*} f  \\
&\geq  (1-\gamma\alpha)^{T-t} \norm{ B_T  \nabla_{\hat{w}^*} f}^2   - \Delta_1 - \Delta_2 - \Delta_3 \\
&\geq  (1-\gamma\alpha)^{T-t} \norm{ B_T  \nabla_{\hat{w}^*} f}^2   - \norm{ \nabla_{\hat{w}^*} f  }^2 O\left( \frac{e^{-\alpha\gamma(T-1)}}{1- e^{-\alpha\gamma}} +   \left(  \gamma (\beta-\alpha)  \right)^{T-t}  \right)
\end{align*}
because $B_T$ has full column rank and 
\begin{align*}
\Delta_1 + \Delta_2 + \Delta_3  
&\leq  \norm{ \nabla_{\hat{w}^*} f  }^2 \left(   \frac{ 2 C_A D }{1- e^{-\alpha\gamma}}   e^{-\alpha\gamma(T-1)}  +   2 C_B D  e^{ - \alpha\gamma\left ( T-1 \right) }  +   \left(  \gamma (\beta-\alpha)  \right)^{T-k} \right) \\
&= \norm{ \nabla_{\hat{w}^*} f  }^2 \times O\left( \frac{e^{-\alpha\gamma(T-1)}}{1- e^{-\alpha\gamma}} +   \left(  \gamma (\beta-\alpha)  \right)^{T-t}   \right)  
\qedhere
\end{align*}

\end{proof}

\section{Proof of Theorem~\ref{th:biased convergence}}

\biasedConvergence*

\begin{proof}
	The proof of this theorem is a standard proof of non-convex  optimization with biased gradient estimates. Here we include it for completeness, as part of it will be used later in the proof of Theorem~\ref{th:convergence of K-step backprop}.
	
	Let $\lambda_\tau$ be the $\tau$th iterate. For short hand, we write  $\d_{\lambda} f_{(\tau)} = \d_{\lambda} f(\lambda_\tau)$, and  $h_{T-K, (\tau)} =  h_{T-K}(\lambda_\tau)$.
	Assume $F$ is $L$-smooth and $\norm{  \d_{\lambda} f_{(\tau)} } \leq G$ and $\norm{h_{T-K, (\tau)}} \leq G$ almost surely for all $\tau$. Then by $L$-smoothness, it satisfies
	\begin{align*}
	F(\lambda_{\tau+1}) \leq F (\lambda_\tau)+ \lr{ \nabla F(\lambda_\tau) }{ \lambda_{\tau+1} - \lambda_{\tau}  } + \frac{L}{2} \norm{\lambda_{\tau+1} - \lambda_{\tau} }^2. 
	\end{align*}
	Let $e_\tau = \d_{\lambda} f_{(\tau)} - h_{T-K, (\tau)}$ be the error in the gradient estimate. Substitute the recursive update $\lambda_{\tau+1} = \lambda_{\tau}  - \eta_t h_{T-K, (\tau)}$ to the above inequality. Conditioned on $\lambda_\tau$, it satisfies 
	\begin{align*}
	\E_{|\lambda_\tau } [ F(\lambda_{\tau+1})]  
	\leq F (\lambda_\tau)+\E_{|\lambda_\tau }  \left[ -\eta_t \lr{ \nabla F(\lambda_\tau) }{  h_{T-K, (\tau)}  } + \frac{L  \eta_t^2}{2} \norm{ h_{T-K, (\tau)} }^2  \right].
	\end{align*}
	Because 
	\begin{align} \label{eq:here comes the bias}
	-\E_{|\lambda_\tau } [ \lr{ \nabla F(\lambda_\tau) }{  h_{T-K, (\tau)}  }]  
	&=   \E_{|\lambda_\tau } \left[- \lr{ \nabla F(\lambda_\tau) }{  \d_{\lambda} f_{(\tau)}  }  +  \lr{ \nabla F(\lambda_\tau) }{ e_\tau } \right]  \nonumber  \\
	&\leq  - \norm{\nabla F(\lambda_\tau) }^2 + G \norm{{ e_\tau }}
	\end{align}
	and
	\begin{align*}
	\frac{1}{2} \norm{ h_{T-K, (\tau)} }^2  =  \frac{1}{2} \norm{ \d_{\lambda} f_{(\tau)} }^2 +  \frac{1}{2} \norm{ e_\tau }^2 - \lr{\d_{\lambda} f_{(\tau)} }{ h_{T-K, (\tau)}  }
	\leq  \frac{3 G^2 }{2}  + \frac{1}{2} \norm{ e_\tau }^2  
	\end{align*}
	we can upper bound $\E_{|\lambda_\tau } [ F(\lambda_{\tau+1})] $ as 
	\begin{align*}
	\E_{|\lambda_\tau } [ F(\lambda_{\tau+1})]  
	\leq  F (\lambda_\tau) + \E_{|\lambda_\tau } \left[ - \eta_\tau \norm{\nabla F(\lambda_\tau) }^2 + \eta_\tau G \norm{{ e_\tau }}  +  L \eta_\tau^2 \left( \frac{3 G^2 }{2}  + \frac{1}{2} \norm{ e_\tau }^2  \right) \right] 
	\end{align*}
	Performing telescoping sum with the above inequality, we have
	\begin{align*}
	\E\left[ \sum_{\tau=1}^{R} \eta_\tau \norm{\nabla F(\lambda_\tau) }^2 \right] &\leq  F (\lambda_1) +    \E\left[ \sum_{\tau=1}^{R} G \eta_\tau  \norm{{ e_\tau }}+   L \eta_\tau^2 \left( \frac{3 G^2 }{2}  + \frac{1}{2} \norm{ e_\tau }^2  \right)  \right] 
	\\ &\leq  F (\lambda_1) +       \sum_{\tau=1}^{R} \left( G \epsilon \eta_\tau    +  \frac{  L (3 G^2 + \epsilon^2)  }{2}   \eta_\tau^2 \right)
	\end{align*}
	Dividing both sides by $\sum_{\tau=1}^{R} \eta_\tau $ and using the facts that $\eta_\tau = O(\frac{1}{\sqrt{\tau}})$
	and that
	\begin{align*}
	\frac{ \sum_{\tau=1}^R \frac{1}{\tau} }{ \sum_{\tau=1}^R \frac{1}{\sqrt{\tau}} } = O\left(\frac{\log R}{\sqrt{R}}\right)
	\end{align*}
	proves the theorem. \qedhere	
\end{proof}

\section{Proof of Theorem~\ref{th:convergence of K-step backprop}}

\convKStep*

\begin{proof}
First we consider the special case when $S$ is deterministic.
Let $H \geq K$. 
We decompose the full gradients into four parts
\begin{align*}
\nabla F  = \d_{\lambda} f = \nabla_\lambda f + q + r + e
\end{align*}
where 
\begin{align*}
q =  \sum_{t=T-K+1}^{T}  B_{t} A_{t+1} \cdots A_{T}  \nabla_{\hat{w}^*} f \\
r =  \sum_{t=T-H+1}^{T-K}  B_{t} A_{t+1} \cdots A_{T}  \nabla_{\hat{w}^*} f \\
e =  \sum_{t=0}^{T-H}  B_{t} A_{t+1} \cdots A_{T}  \nabla_{\hat{w}^*} f 
\end{align*}
We assume that $w_t$ enters a locally strongly convex region for $t\geq H$. This implies, by Proposition~\ref{pr:exp convergence}, that $\norm{e} \leq O( e^{-\alpha\gamma H} \norm{\nabla_{\hat{w}^*} f} )$.

To prove the theorem, we first verify two conditions: 
\begin{enumerate}
\item By Lemma~\ref{lm:bound of many terms}, the assumption $ \nabla_\lambda f^\top (\d_\lambda f + h_{T-K} -\nabla_\lambda f )  
\geq \Omega(\norm{\nabla_\lambda f}^2)$, and  $\norm{e} \leq O( e^{-\alpha\gamma H} \norm{\nabla_{\hat{w}^*} f} )$:
\begin{align*}
\d_{\lambda} f^\top h_{T-K} &=  (\nabla_\lambda f + q + r + e)^\top (\nabla_\lambda f + q) \\
&= \norm{\nabla_\lambda f}^2  + \nabla_\lambda f^\top (q + e + r)  + q^\top \nabla_{\lambda} f +  q^\top (q + r) + q^\top e \\
&\geq \Omega(\norm{\nabla_\lambda f}^2) + q^\top (q + r) + q^\top e &\text{(Assumption)}
\\
&\geq  \Omega( \norm{\nabla_\lambda f}^2) + \Omega( \norm{\nabla_{\hat{w}^*} f }^2) + q^\top e   & \text{(Lemma~\ref{lm:bound of many terms})} \\
&\geq  \Omega(
\norm{\nabla_\lambda f}^2) + \Omega( \norm{\nabla_{\hat{w}^*} f }^2)  - O\left(e^{-\alpha\gamma H}  \norm{\nabla_{\hat{w}^*} f }^2 \right)  &\text{($\norm{e} \leq O( e^{-\alpha\gamma H} \norm{\nabla_{\hat{w}^*} f} )$)}
\end{align*}
where we note 
\begin{align*}
\d_\lambda f + h_{T-K} -\nabla_\lambda f  
&=
\nabla_\lambda f + q + r + e + \nabla_\lambda f + q - \nabla_\lambda f\\
&= \nabla_\lambda f + q + r + e  + q 
\end{align*}

%\begin{align*}
%\d_{\lambda} f^\top h_{T-K} &=  (\nabla_\lambda f + q + r + e)^\top (\nabla_\lambda f + q) \\
%&= \norm{\nabla_\lambda f}^2  + \nabla_\lambda f^\top (q + e + r)  + q^\top \nabla_{\lambda} f +  q^\top (q + r) + q^\top e \\
%&\geq \Omega(\norm{\nabla_\lambda f}^2) + q^\top (q + r) + q^\top e &\text{(Assumption)}
%\\
%&\geq  \Omega( \norm{\nabla_\lambda f}^2) + \Omega( \norm{\nabla_{\hat{w}^*} f }^2) + q^\top e   & \text{(Lemma~\ref{lm:bound of many terms})} \\
%&\geq  \Omega(
%\norm{\nabla_\lambda f}^2) + \Omega( \norm{\nabla_{\hat{w}^*} f }^2)  - O\left(e^{-\alpha\gamma H}  \norm{\nabla_{\hat{w}^*} f }^2 \right)  &\text{($\norm{e} \leq O( e^{-\alpha\gamma H} \norm{\nabla_{\hat{w}^*} f} )$)}
%\end{align*}
Therefore, for $H$ large enough, it holds that 
\begin{align} \label{eq:sufficient decrease condition}
\d_{\lambda} f^\top h_{T-K}  \geq \Omega( \norm{\nabla_\lambda f}^2 +  \norm{\nabla_{\hat{w}^*} f }^2)  
\end{align}
\item By definition of $h_{T-K} = \nabla_\lambda f + q $,  it holds that 
\begin{align} \label{eq:not-so-large condition}
\norm{h_{T-K}}^2 \leq 2\norm{\nabla_\lambda f}^2 + 2 \norm{q}^2 \leq O( \norm{\nabla_\lambda f}^2 + \norm{\nabla_{\hat{w}^*} f}^2 ) 
\end{align}
\end{enumerate}

Next, we prove a lemma
\begin{lemma} \label{lm:convergence of nonconvex optimization}
Let $f$ be a lower-bound and $L$-smooth function. Consider the iterative update rule
\begin{align*}
	x_{t+1} = x_t - \eta g_t
\end{align*}
where $g_t$ satisfies $g_t^\top \nabla f(x_t) \geq c_1 h_t^2 $ and $\norm{g_t}^2 \leq c_2  h_t^2$, for some constant $c_1, c_2 >0$ and scalar $h_t$. 
Suppose $f$ is lower-bounded and $\eta$ is chosen such that $\left( - c_1 \eta + \frac{L c_2\eta^2 }{2} \right) \leq 0 $. Then $\lim\limits_{t\to\infty} h_t = 0$.
\end{lemma}
\begin{proof}
By $L$-smoothness,
\begin{align*}
f(x_{t+1}) - f(x_t) &\leq \nabla f(x_t)^\top (x_{t+1} - x_t) + \frac{L}{2} \norm{x_{t+1} - x_t}^2 \\
&=  - \eta \nabla f(x_t)^\top g_t + \frac{L \eta^2 }{2} \norm{g_t}^2 \\
&\leq \left( - c_1 \eta + \frac{L c_2 \eta^2 }{2} \right) h_t^2
\end{align*}
By telescoping sum, we can show
$
\sum_{t=0}^{\infty} \left( c \eta - \frac{L \eta^2 }{2} \right) h_t^2 < \infty
$, 
which implies $\lim_{t\to\infty} h_t  = 0$.
\end{proof}

Finally, we prove the main theorem by applying Lemma~\ref{lm:convergence of nonconvex optimization}. 
Consider a deterministic problem. % (the particular condition). 
Take $h_t^2 = \norm{\nabla_\lambda f (\lambda_t)}^2 + \norm{\nabla_{\hat{w}^*} f (\lambda_t) }^2 $. 
Because of~\eqref{eq:sufficient decrease condition} and~\eqref{eq:not-so-large condition}, by Lemma~\ref{lm:convergence of nonconvex optimization}, it satisfies that 
\begin{align*}
\lim_{t\to\infty} h_t  =  \lim_{t\to\infty}  \norm{\nabla_\lambda f (\lambda_t)}^2 + \norm{\nabla_{\hat{w}^*} f (\lambda_t) }^2  = 0
\end{align*}
As $\norm{\d_\lambda f } \leq  O(\norm{\nabla_\lambda f} + \norm{\nabla_{\hat{w}^*} f })$, it shows $\norm{\d_\lambda f }$ converges to zero in the limit. %, which concludes the proof.

%For stochastic problem, we can combine the proof strategy in Lemma~\ref{lm:convergence of nonconvex optimization} and Theorem~\ref{th:biased convergence}. To remove the bias, we need a stronger property to replace~\eqref{eq:here comes the bias}. This is done by generalizing the condition~\eqref{eq:sufficient decrease condition} to hold in expectation for stochastic problems.
%
%
% is to show that, conditioned on the current 
%\begin{align*}
% \d_{\lambda} F^\top \E[ h_{T-K} ]
%\end{align*}
%
%
% and show that
%\begin{align*}
%\textstyle
%\E\left[ \sum_{\tau=0}^{\TT} \eta_{\tau} h_t^2 / \sum_{\tau=0}^{\TT} \eta_{\tau} \right] 
%\leq  
%\widetilde O\left(    (\epsilon^2+1)/ \sqrt{\TT}\right).
%\end{align*}	 
%Again using  $\norm{\d_\lambda f }^2 \leq  O(h_t^2)$, the conclusion follows.
\end{proof}

\section{Proof of Theorem~\ref{th:counterexample}} \label{app:proof of counterexample}
\counterExample*
\begin{proof}
We prove the non-convergence using the following strategy. First we show that, when assumption 3 in Theorem~\ref{th:convergence of K-step backprop}, i.e.
\begin{align} \label{eq:interfering free assumption}
		 \nabla_\lambda f^\top (\d_\lambda f + h_{T-K} -\nabla_\lambda f )  \geq \Omega(\norm{\nabla_\lambda f}^2)
\end{align}
does not hold, there is some problem such that $h_{T-k} \neq 0$ for all stationary points (i.e. $\lambda$ such that $\d_\lambda f = 0$). Then we show that, for such a problem, optimizing $\lambda$ with $h_{T-k}$ cannot converge to any of the stationary points.

\paragraph{Counter example}
To construct the counterexample, we consider a scalar deterministic bilevel optimization problem of the form
\begin{align}   \label{eq:counterexample}
\begin{split}
&\min_{\lambda } 
 \frac{1}{2} (\hat{w}^*)^2 + \phi(\lambda)    \\
\text{s.t.} \quad &  \hat{w}^* \approx w^* \in \argmin_{w } \frac{1}{2} (w - \lambda)^2
\end{split}
\end{align}
in which $\phi$ is some perturbation function that we will later define, and 
  $\hat{w}^*$ is computed by performing $T>1$ steps of gradient descent in the lower-level optimization problem with some constant initial condition $w_0$ and  constant step size $0<\gamma<1$, i.e. 
\begin{align*}
\hat{w}^* = w_T, \qquad  w_{t+1}  = w_t - \gamma (w_t - \lambda)
\end{align*}

We can observe this problem satisfies \emph{almost} all the assumptions in Theorem~\ref{th:convergence of K-step backprop}:
\begin{enumerate}
\item The lower-level objective  $g$ is smooth and strongly convex. (Proposition~\ref{pr:exp convergence})
\item The upper-level objective $F$ is smooth. (Theorem~\ref{th:biased convergence})
\item The lower-level objective $g$ is second-order continuously differentiable (assumption 1 in Theorem~\ref{th:convergence of K-step backprop})
\item The Jacobian if full rank, i.e. $B_t=\gamma >0$  (assumption 2 in Theorem~\ref{th:convergence of K-step backprop})
\item The upper-level objective function  is deterministic, i.e. $F=f$  (assumption 4 in Theorem~\ref{th:convergence of K-step backprop})
\end{enumerate}
But we will show that properly setting $\phi$ can break the non-interfering assumption in~\eqref{eq:interfering free assumption} (i.e. assumption 3 in Theorem~\ref{th:convergence of K-step backprop}) and then creates a problem such that optimizing $\lambda$ with $K$-RMD does not converge to an exact stationary point.

We follow the two-step strategy mentioned above.
\paragraph{Step 1: Non-vanishing approximate gradient}
Without loss of generality, let us consider optimizing $\lambda$ with $1$-RMD. In this case we can write the approximate and the exact gradients in closed form as 
\begin{align} \label{eq:gradients in closed form}
%\begin{split}
h_{T-1} = \nabla\phi+ w^*\gamma, \qquad 
\d_\lambda f =  \nabla\phi+ w^*\gamma \sum_{t=0}^{T} (1-\gamma)^{T-t}
%\end{split}
\end{align}
which are given by~\eqref{eq:unrolling} and~\eqref{eq:incomplete RMD}. We will show that by properly choosing $\phi$, we can define $f(\lambda) = \frac{1}{2} (\hat{w}^*)^2 + \phi(\lambda)$ such that, at any of the stationary points of $f$, the approximate gradient of $1$-RMD does not vanish. That is, we show when $\d_\lambda f =0$, $h_{T-1} \neq 0$.

Before proceeding, let us define $u = w^*\gamma$ and $v =  w^*\gamma \sum_{t=0}^{T} (1-\gamma)^{T-t}$ for convenience. 
To show how to construct $\phi$, let us consider the stationary points in the case\footnote{Note in this special case, assumption 3 in Theorem~\ref{th:convergence of K-step backprop} holds trivially when $\phi(\lambda)=0$ (i.e. $\nabla_{\lambda} f = 0$) and optimizing $\lambda$ with $K$-RMD converges to an exact stationary point. 
} when $\phi=0$. Let $P_0$ denote the set of these stationary points, i.e. $P_0 = \{ \lambda : v=0  \}$. Since $f$ is smooth and lower-bounded, we know that $P_0$ is non-empty, and from the construction of our counterexample we know that $P_0$ contains exactly the $\lambda$s such that $w^* = 0$. 

This implies that for $\lambda \in \R \backslash P_0$, it satisfies $w^* \neq 0$ and therefore 
\begin{align} \label{eq:uv inequality}
uv = ( w^* \gamma )^2  \sum_{t=0}^{T} (1-\gamma)^{T-t} > 0
\end{align}
We use this fact to pick an adversarial $\phi$. Consider any smooth, lower-bounded $\phi$ whose stationary points are not in $P_0$, e.g. $\phi(\lambda) = \frac{1}{2}(\lambda - \lambda_0)^2$ and $\lambda_0 \notin P_0$. Then $f(\lambda) = \frac{1}{2}(\hat{w}^*)^2 + \phi(\lambda) $ has a non-empty set of stationary points $P_\phi$ such that $P_\phi \cap P_0 = \emptyset$.
We see that, for such $\phi$,  the non-interfering assumption (assumption 3 in Theorem~\ref{th:convergence of K-step backprop}) is violated in $P_\phi$: 
\begin{align*}
\nabla_\lambda f^\top (\d_\lambda f + h_{T-1} -\nabla_\lambda f )
&=\nabla_\lambda f^\top ( \nabla_\lambda f + u -\nabla_\lambda f )
&\text{$\because \d_\lambda f=0$ and $h_{T-1}=\nabla_{\lambda} f + u$}
	 \\
&=\nabla_\lambda \phi^\top  u \\
&= - v u &\text{$\because 0=\d_\lambda f = \nabla_{\lambda}\phi + v$}\\
&<0 &\text{$\because \eqref{eq:uv inequality}$ and  $P_\phi \cap P_0 = \emptyset$}\\
&<  (\nabla_{\lambda}\phi )^2  &\text{$\because v>0$ for $\lambda \in P_\phi$}
  %\geq \Omega(\norm{\nabla_\lambda f}^2)
\end{align*}
And we show for any $\lambda \in P_\phi$ it holds that $h_{T-1} \neq 0$. This can be seen from the definition
\begin{align*}
h_{T-1} = \nabla\phi+ u = \d_\lambda f + u - v = u - v \neq 0
\end{align*}
where the last inequality is because $w^*\neq 0 $ for  $\lambda \in P_\phi$. 

\paragraph{Step 2: Non-convergence to any stationary point}
We have shown that there is a problem which satisfies all the assumptions but assumption 3 of Theorem~\ref{th:convergence of K-step backprop}, and at any of its stationary points (i.e. when $\d_\lambda f=0$) we have $h_{T-K}\neq 0$. 
Now we show this property implies failure to converge to the stationary points for the general problems considered in Theorem~\ref{th:counterexample} (i.e. we do not rely on the form made in Step 1 anymore).

We prove this by contradiction. Let $\lambda^*$ be one of the stationary points. We choose $\delta_0>0$ such that, for some $\epsilon >0$, $\norm{h_{T-K}} > \epsilon/\gamma$ for all $\lambda $ inside the neighborhood $\{ \lambda: \norm{\lambda - \lambda^*} < \frac{\delta_0}{2}\}$, where we recall $\gamma$ is the step size of the lower-level optimization problem.
 A non-zero $\delta_0$ exists because $h_{T-1}$ is continuous by our assumption and $h_{T-K} \neq 0$ at $\lambda^*$. 

We are ready to show the contradiction. Let $\delta =\min \{ \delta_0, \epsilon \}$.  Suppose there is a sequence $\{\lambda_\tau\}$ that converges to the stationary point $\lambda^*$. This means that there is $0<M<\infty$ such that,  $\forall \tau\geq M$, $\norm{\lambda_{\tau} - \lambda^*} < \frac{\delta}{2}$, which implies that  $\forall \tau \geq M$, $\norm{\lambda_{\tau+1} - \lambda_{\tau}} < \delta$. However, by our choice of $\delta_0$, $\norm{\lambda_{\tau+1} - \lambda_{\tau}} = \gamma \norm{h_{T-K}} > \epsilon \geq \delta$, leading to a contradiction. 

Thus, no sequence $\{\lambda_\tau\}$ converges to any of the stationary points. This concludes our proof.
\end{proof}

\section{Proof of Proposition~\ref{pr:taylor}}

\matrixTaylor*

\begin{proof}
  Recall our shorthand that $\nabla_{\lambda,w}g$ and $\nabla_{w,w}g$ are evaluated at $(w^*, \lambda)$. In the limit, it holds that 
\begin{align*}
  \lim_t A_{t} &=  \lim_t \nabla_{w} \Xi_{t}(w_{t-1}, \lambda) =   \nabla_{w}( w^* - \gamma \nabla_{w} g(w^*,\lambda) ) = I - \gamma \nabla_{w,w} g \eqqcolon A_{\infty}\\
  \lim_t B_{t} &=  \lim_t \nabla_{\lambda} \Xi_{t}(w_{t-1}, \lambda) =   \nabla_{\lambda}( w^* - \gamma \nabla_{w} g(w^*,\lambda) ) =  - \gamma \nabla_{\lambda,w} g \eqqcolon B_{\infty}
\end{align*}
To prove the equality~\eqref{eq:infinte series}, we use Lemma~\eqref{lm:taylor series}. 
\begin{lemma} \label{lm:taylor series}
	\citep{horn1990matrix} For a matrix $A$ with $\norm{A} < 1$, it satisfies that 
	\begin{align*}
	(I-A)^{-1} = \sum_{k=0}^{\infty} A^k
	\end{align*}
\end{lemma}
Since $\gamma \leq \frac{1}{\beta}$, we have $\gamma\alpha I \preceq \gamma \nabla_{w,w} g \preceq I$, so $\norm{I-\gamma \nabla_{w,w} g} < 1$. By Lemma~\ref{lm:taylor series},
\begin{align*}
\nabla_{w,w}^{-1} g  =  \gamma (I - I + \gamma\nabla_{w,w} g )^{-1} =  \gamma \sum_{k=0}^{\infty} (I - \gamma \nabla_{w,w} g)^k = \gamma \sum_{k=0}^{\infty} A_{\infty}^k  
\end{align*}
Therefore, 
\begin{align*}
- \nabla_{\lambda, w} g \nabla_{w, w}^{-1} g  =  \left( - \gamma \nabla_{\lambda, w}g \right) \left( \frac{1}{\gamma}  \nabla_{w, w}^{-1} g \right) = B_\infty  \sum_{k=0}^{\infty} A_{\infty}^k  
\end{align*}

\end{proof}

\section{Detailed experimental setup}
In this appendix, we provide more details about the settings we used in each experiment. We use Adam~\citep{kingma2014adam} to
optimize the upper-level objective and vanilla gradient descent for the lower objective. We denote by $\hat{w}^*$ the results of 
running $T$ steps of gradient descent with step size $\gamma$.

\subsection{Data hypercleaning}
\label{sec:mnist_appendix}

In this appendix, we provide more details about the data hypercleaning experiment on MNIST from Section \ref{sec:mnist}.

Both the training and the validation sets consist of 5000 class-balanced examples from the MNIST dataset. The test set consists of the remaining examples. For each training example, with probability $\frac12$, we replaced the label with a uniformly random one.

For various $K$, we performed $K$-RMD for 1000 hyperiterations. Like in the toy experiment (Section \ref{sec:toy}) we adjusted the initial meta-learning rate $\eta_0$ for each $K$ so that the norm of the initial update was roughly the same for each $K$.

We asserted earlier that the reported F1 scores are not sensitive to our choice of threshold $\lambda_i < -3$.
To validate this assertion, we repeated the experiment for various thresholds. F1 scores are reported in the table below.

\begin{table}
\centering
\begin{tabular}{c|ccc}
$K$ & $\lambda_i < -4$ & $\lambda_i < -3$ & $\lambda_i < -1$ \\
\hline
1   & 0.84 & 0.84 & 0.84 \\
5   & 0.89 & 0.89 & 0.90 \\
25  & 0.89 & 0.89 & 0.89 \\
50  & 0.89 & 0.89 & 0.89 \\
100 & 0.89 & 0.89 & 0.89 \\
\end{tabular}
\end{table}
We only ran these experiments for $150$ hyperiterations, because the F1 score has essentially converged by that point.
Indeed, the plot below shows identification of corrupted labels for $K=1$, with cutoff $\lambda_i < -4$.
The X axis is in units of 1000 hyperiterations.
We see that $1$-RMD rapidly identifies most of the mislabeled examples, with a few false positives.

\begin{figure}
\centering
\includegraphics[width=0.4\textwidth]{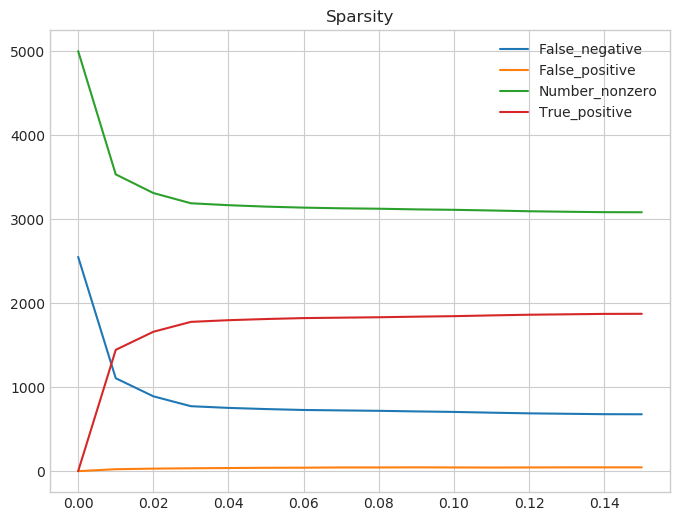}
\end{figure}

\subsection{Task interaction}
\label{sec:cifar_appendix}
We use $T=100$ iterations of gradient descent with learning rate $0.1$ in the lower objective which yields $\hat{w}_S^*$.
To ensure that $C$ is symmetric, and that $C_{ij}$ and $\rho$ are nonnegative, we re-parametrize them as 
$\rho = \text{softplus}(\nu)$ and $C = A+A^\top$, where $A_{ij} = \text{softplus}(B_{ij})$ and $B$ is a hyperparameter matrix.
Thus, the hyperparameters to be optimized are $\lambda=\{B, \nu\}$.

Rather than using raw pixels, we extract image features from the output of the average pooling layer in Resnet-$18$~\citep{he2016deep} which is trained on ImageNet~\citep{deng2009imagenet}.
We use the same data pre-processing that is used for training Resnet architecture.

When reporting test accuracy, we run 10 independent trials.
In each trial, we sample the training and validation datasets with a balanced set of $m$ examples each ($m = 50$ for CIFAR-10 and $m = 300$ for CIFAR-100)
and use the rest of the dataset for testing.
To avoid over-fitting, we use early stopping when the testing error does not improve for $500$ hyper-iterations.

Although we are using a similar setting as~\citet{franceschi2017forward}, our results on full back-propagation are quite different from theirs.
We believe it is because we are using a different network architecture and pre-processing method for feature extraction.

\subsection{One-shot classification}
\label{app:one_shot}
\paragraph{Dataset}
The Omniglot dataset \citep{lake2015human}, a popular benchmark for few-shot learning, is used in this experiment. We consider $5$-way classification with $1$ training and $15$ validation examples for each of the five classes.
To evaluate the generalization performance, we restrict the meta-training dataset to a random subset of $1200$ of the $1623$ Omniglot characters. 
The meta-validation dataset consists of $100$ other characters, and meta-testing dataset has the remaining $323$ characters. We use the meta-validation dataset for tuning the upper-level optimization parameters and report the performance of the algorithm on the meta-testing dataset. Note that no data augmentation method is used in the training. 

\paragraph{Neural Network and Optimization}
The overall neural network architecture is shown in Figure~\ref{fig:one_shot_nn}. Our architecture inherits the hyper-representation model of~\citet{franceschi2017bridge} with some modifications. The first two convolutional layers, parametrized by hyperparameter $\lambda = \{\lambda_{l_1}, \lambda_{l_2}\}$,
transform the input image into a ``hyper-representation'' space. The last three layers, parametrized by $w = \{w_{l_3}, w_{l_4}, w_{l_5}\}$ are fine-tuned in the lower-level optimization.
Additionally, we have regularization hyperparameters $\lambda_r = \{\rho_i\}_{i=1}^3 \cup \{c_j\}_{j=1}^3$.
The overall setup corresponds essentially to meta-learning the two bottom layers of a CNN; for each task, the weights in the first two layers are frozen, and the $k$-way classifier of the last three layers is fine tuned. Overall, the model has $\approx 110$k hyperparameters and $\approx 75$k parameters.

We use a meta-batch-size of $4$ in each hyper-iteration.
To limit the training time, we stop all the algorithms after $5000$ hyper-iterations. 
Needless to say, these results could be further improved by using data augmentation, higher meta-batch size, and running more hyper-iterations.
However, our current setup is selected so that all the experiments can be run in a reasonable amount of time, while sharing a similar setting used in practical one-shot learning.

\begin{figure}
  \centering
  \includegraphics[width=.7\linewidth]{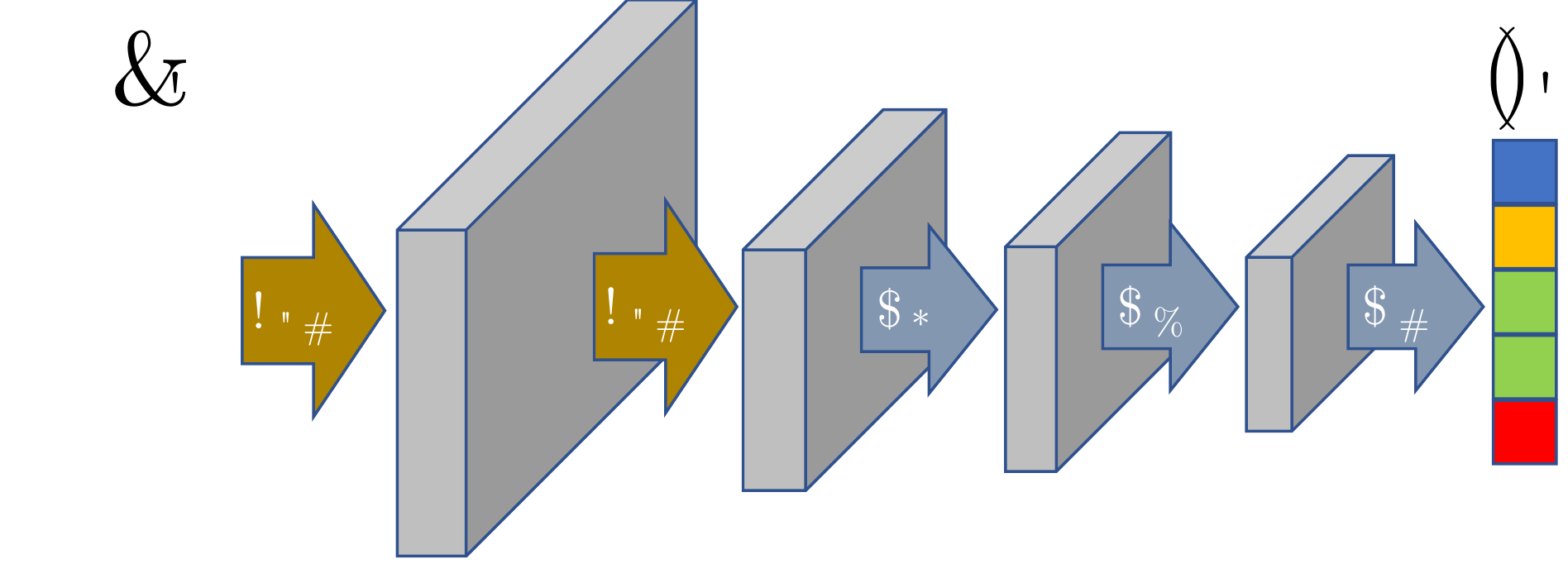}%
  \caption{\label{fig:one_shot_nn}One-shot learning network architecture. The first two convolutional layers map the input image into a ''hyper-representation'' space which is frozen while optimizing the lower-level objective.
The last three layers are tuned for each task and regularized to avoid overfitting.
All the convolutional layers have $64$ $3\times3$ kernels.
There is a max-pooling layer followed by a batch-normalization and a ReLU layer after each convolution.}%
\end{figure}

\end{document}